\documentclass[nohyperref]{article}

\usepackage{microtype}
\usepackage{graphicx}
\usepackage{subfigure}
\usepackage{booktabs} 

\usepackage{hyperref}

 \usepackage[accepted]{icml2022}

\usepackage{amsmath}
\usepackage{amssymb}
\usepackage{mathtools}
\usepackage{amsthm}

\usepackage[capitalize,noabbrev]{cleveref}

\theoremstyle{plain}
\newtheorem{theorem}{Theorem}[section]
\newtheorem{proposition}[theorem]{Proposition}
\newtheorem{lemma}[theorem]{Lemma}

\theoremstyle{definition}

\newtheorem{assumption}[theorem]{Assumption}
\theoremstyle{remark}
\newtheorem{remark}[theorem]{Remark}

\usepackage[textsize=tiny]{todonotes}

\icmltitlerunning{Langevin Monte Carlo for Contextual Bandits}

\usepackage{mmll}
\usepackage{natbib}
\usepackage{enumitem}

\def \hbtheta{\widehat{\btheta}}
\def \tbtheta{\widetilde{\btheta}}
\def \dd{\text{d}}

\def \algname{\text{LMC-TS}}

\def \tO{\widetilde{O}} 

\begin{document}

\twocolumn[
\icmltitle{
Langevin Monte Carlo for Contextual Bandits
}



\icmlsetsymbol{equal}{*}

\begin{icmlauthorlist}
\icmlauthor{Pan Xu}{caltech}
\icmlauthor{Hongkai Zheng}{caltech}
\icmlauthor{Eric Mazumdar}{caltech}
\icmlauthor{Kamyar Azizzadenesheli}{purdue}
\icmlauthor{Anima Anandkumar}{caltech}
\end{icmlauthorlist}

\icmlaffiliation{caltech}{Department of Computing and Mathematical Sciences, California Institute of Technology, Pasadena, CA, USA}
\icmlaffiliation{purdue}{Department of Computer Science, Purdue University, West Lafayette, IN, USA}

\icmlcorrespondingauthor{Pan Xu}{pan.xu@duke.edu}

\icmlkeywords{Machine Learning, ICML}

\vskip 0.3in
]



\printAffiliationsAndNotice{}  

\begin{abstract}
We study the efficiency of Thompson sampling for contextual bandits.  Existing Thompson sampling-based algorithms need to construct a Laplace approximation (i.e., a Gaussian distribution) of the posterior distribution, which is inefficient to sample in high dimensional applications for general covariance matrices. Moreover, the Gaussian approximation may not be a good surrogate for the posterior distribution for general reward generating functions. We propose an efficient posterior sampling algorithm, viz., Langevin Monte Carlo  Thompson Sampling (LMC-TS), that uses Markov Chain Monte Carlo (MCMC) methods to directly sample from the posterior distribution in contextual bandits. Our method is computationally efficient since it only needs to perform noisy gradient descent updates without constructing the Laplace approximation of the posterior distribution. We prove that the proposed algorithm achieves the same sublinear regret bound as the best Thompson sampling algorithms for a special case of contextual bandits, viz., linear contextual bandits. We conduct experiments on both synthetic data and real-world datasets on different contextual bandit models, which demonstrates that directly sampling from the posterior is both computationally efficient and competitive in  performance. 
\end{abstract}

\section{Introduction}

A bandit problem is a sequential decision-making problem wherein an agent, in each round,  observes an action set, chooses an action (or arm) from the set, and then observes a reward from the environment. A bandit learning algorithm aims to learn a policy for the agent to maximize its cumulative rewards based on its historical observations of actions and rewards. In the vast majority of real-world applications, each arm is usually associated with side information in the form of a feature or context vector that describes the arm. The mean reward for an arm is expressed as some unknown function of the arm's feature vector and an unknown weight parameter that is shared across different arms. This setting---known as the contextual bandit problem--- has been extensively studied in the literature~\citep{langford2007epoch,chu2011contextual,abbasi2011improved,agrawal2013thompson,filippi2010parametric,li2017provably,lale2019stochastic,kveton2020randomized}. 

The main challenge in contextual bandit problems is addressing the well known exploitation versus exploration trade-off, which requires a careful balance between choosing the myopically better arm and choosing an under-sampled worse arm. Existing algorithms for maximizing the cumulative reward in bandit problems mainly follow either one of the following two algorithmic frameworks. The first framework follows the principle of optimism in the face of uncertainty (OFU), and algorithms designed using such ideas have been widely applied to both finite armed bandits, also known as multi-armed bandits (MAB)~\citep{auer2002finite,menard2017minimax}, and contextual bandits~\citep{chu2011contextual,abbasi2011improved,li2017provably,zhou2020neural,xu2022neural}. The second dominant category of bandit algorithm makes use of the idea of Thompson or posterior sampling~\citep{thompson1933likelihood}. Such algorithms have been widely used in practice due to their ease of implementation and impressive empirical performance, and have only recently started to be well understood theoretically in multi-armed bandits~\citep{agrawal2012analysis,kaufmann2012thompson,russo2014learning,jin2021mots} and contextual bandits~\citep{chapelle2011empirical,agrawal2013thompson,riquelme2018deep,wang2020thompson,zhang2021neural}. 

One crucial area in which the two types of algorithms differ is in their ease of implementation. In contextual bandit problems, algorithms based on the OFU principle usually need to solve a bi-linear optimization problem, making them computationally expensive to implement outside of simple problems despite coming with stronger theoretical guarantees. In contrast, Thompson sampling algorithms only need to solve a linear optimization problem over the arm set since the uncertainty in the posterior distribution automatically accounts for the exploration in the parameter space. Furthermore, Thompson sampling has been observed to be empirically competitive to---or sometimes even better---than OFU based algorithms~\citep{chapelle2011empirical}.

Most existing Thompson sampling algorithms first construct a Laplace approximation (which is essentially a Gaussian distribution)~\citep{chapelle2011empirical} of the underlying posterior distribution on the data and then sample from the Gaussian distribution to explore in the parameter space. The Laplace approximation in Thompson sampling usually leads to a non-isotropic covariance matrix. It is well known that sampling from a Gaussian distribution with a general covariance matrix is usually computationally expensive in high dimensional applications. Moreover, when the reward generating function is nonlinear with respect to the weight parameter such as in generalized linear bandits~\citep{kveton2020randomized} and neural contextual bandits~\citep{riquelme2018deep,zhang2021neural}, the true posterior distribution may not be well approximated by a Gaussian distribution and thus the Laplace approximation could be a poor surrogate for the posterior distribution. 

\textbf{Our approach:} In this paper, we propose  an algorithm, viz., Langevin Monte Carlo Thompson sampling ($\algname$), which directly samples from the data posterior distribution instead of a Laplace approximation in contextual bandits. In particular, by incorporating Langevin Monte Carlo~\citep{bakry2014analysis}, our algorithm only needs to perform noisy gradient descent updates, which can generate samples that provide a good approximation of the posterior with arbitrary accuracy if it is run for sufficiently many steps. This is contrast with Laplace approximation for Thompson sampling~\citep{chapelle2011empirical}, which has a fixed approximation error for the posterior distribution and thus the covariance matrix needs to be carefully redesigned in different contextual bandit problems to achieve reasonable performance~\citep{chapelle2011empirical,kveton2020randomized,riquelme2018deep,zhang2021neural}. Moreover, due to the simplicity of noisy gradient descent updates, the proposed algorithm is directly applicable to many bandit problems where deep neural network function classes are used. 

\textbf{Contributions} of this paper are summarized as follows:
\begin{itemize}[leftmargin=*]
    \item We propose a practical and efficient bandit algorithm $\algname$, which only needs to perform noisy gradient descent updates to approximately sample from the data posterior distribution.
    $\algname$ is easily implementable and scalable to large-scale and high dimensional problems including deep learning applications. It also works simultaneously for a large class of contextual bandit models including linear contextual bandits, generalized linear bandits, and neural contextual bandits.
    \item We theoretically prove that $\algname$ achieves a $\widetilde O(d\sqrt{dT})$ regret for linear contextual bandits, where $d$ is the dimension of the problem and $T$ is the time horizon. This result matches the best regret bound for Thompson sampling algorithms in linear contextual bandits~\citep{agrawal2013thompson}.
    \item We further conduct thorough experiments on both synthetic datasets and real-world datasets (UCI machine learning datasets and a high dimensional image dataset CIFAR10) to show that one algorithm is enough for learning many different complex bandit models by comparing it with different baseline algorithms in linear contextual bandits, generalized bandits, and neural contextual bandits respectively. 
\end{itemize}

\paragraph{Notation} We use $[k]$ to denote a  set $\{1,\ldots,k\}$, $k\in\NN^+$. $\|\xb\|_2=\sqrt{\xb^{\top}\xb}$ is the Euclidean norm of a vector $\xb\in\RR^d$. For a matrix $\Vb\in\RR^{m\times n}$, we denote by $\|\Vb\|_2$ and $\|\Vb\|_F$ its operator norm and Frobenius norm respectively. For a semi-positive definite matrix $\Vb\in\RR^{d\times d}$ and a vector $\xb\in\RR^d$, we denote the Mahalanobis norm as $\|\xb\|_{\Vb}=\sqrt{\xb^{\top}\Vb\xb}$. For an event $E$ on a probability space, we denote $E^c$ as its complement event such that $\PP(E)+\PP(E^c)=1$. For a function $f(T)$, we use the common big O notation  $O(f(T))$ to hide constant factors with respect to $T$ and use $\tO(f(T))$ to omit the logarithmic dependence on $T$.

\section{Preliminary}\label{sec:preliminary}
\paragraph{Contextual Bandits} Contextual bandits are a wide class of sequential decision problems, where the player makes the decision based on an observation of an action set consisting of feature vectors as contexts for different actions. In particular, at round $t$, the player observes an action set $\cX_t\subseteq\RR^d$, and chooses an arm or action which is represented by a feature vector $\xb_t\in\cX_t$. Note that in this paper we do not assume the action set is finite nor is fixed in each round. Then a reward $r_t$ is immediately revealed to the agent by the environment. In contextual bandit problems, it is often assumed that the mean reward of an action with feature $\xb\in\RR^d$ is given by a reward generating function $f(\xb,\btheta^*)$ and the observed reward is $r(\xb)=f(\xb,\btheta^*)+\xi$, where $\btheta^*\in\RR^{d'}$ is an unknown weight parameter that is shared across all arms, and $\xi$ is a random noise incurred in the observation. For instance, in linear contextual bandits~\citep{chu2011contextual,abbasi2011improved,agrawal2013thompson}, we have $\btheta^*\in\RR^d$ and $f(\xb,\btheta^*)=\xb^{\top}\btheta^*$; in generalized linear bandits~\citep{filippi2010parametric,li2017provably,kveton2020randomized,ding2021efficient}, we have $f(\xb,\btheta^*)=\mu(\xb^{\top}\btheta^*)$ for some link function $\mu(\cdot)$; and for neural contextual bandits~\citep{riquelme2018deep,zhou2020neural,zhang2021neural,xu2022neural}, $f(\xb,\btheta^*)$ is a neural network, where $\btheta^*$ is the concatenation of all weight parameters  and $\xb$ is the input.

The objective of a bandit algorithm is to maximize the cumulative rewards over a time horizon $T$, which is equivalent to minimizing the following pseudo regret~\citep{lattimore2020bandit}:
\begin{align}\label{def:pseudo_regret}
    R(T)=\EE\Bigg[\sum_{t=1}^{T}(r(\xb_t^*)-r(\xb_t))\Bigg],
\end{align}
where $\xb_t\in\cX_t$ is the arm chosen by the bandit algorithm at round $t$, and $\xb_t^*=\argmax_{\xb\in\cX_t} \EE[r(\xb)]$ is the arm with the maximum expected reward at round $t$. Note that this definition of regret is based on the best oracle arm $\xb_t^*$, which is more general than the definition based on the best policy achievable within a predefined policy class in adversarial bandits~\citep{bubeck2012regret}.

\paragraph{Laplace Approximation Thompson Sampling}

Among the most popular bandit algorithms, Thompson sampling~\citep{thompson1933likelihood,chapelle2011empirical,russo2018tutorial} is known to be simple and efficient in practice, which uses a Laplace approximation to approximate the posterior distribution of the data. After $t-1$ rounds of the bandit problem, assume we have collected data $\{\xb_1,r_1,\xb_2,r_2,\ldots,\xb_{t-1},r_{t-1}\}$.  Define the following quantities based on historical data.
\begin{align}\label{eq:V_b_update}
    \Vb_t=\lambda\Ib+\sum_{s=1}^{t-1}\xb_{s}\xb_{s}^{\top}, \quad\bbb_t=\sum_{s=1}^{t-1}r_{s}\xb_{s},
\end{align}
where $\lambda>0$ is a regularization parameter. Denote $\widehat\btheta_t=\Vb_t^{-1}\bbb_t$. At round $t$, the agent receives an action set $\cX_t\subseteq\RR^d$ which consists of feature vectors of candidate actions at round $t$. Then linear Thompson sampling (LinTS)~\citep{agrawal2013thompson} samples a parameter $\widetilde\btheta_t$ from distribution $\cN(\hbtheta_t,v_t\Vb_t^{-1})$ and then chooses the arm as follows $\xb_t=\argmax_{\xb\in\cX_t}\xb^{\top}\tbtheta_t$. After that, it observes the reward $r_t$ for round $t$. Based on newly collected action feature $\xb_t$ and reward $r_t$, the quantities in \eqref{eq:V_b_update} can be updated and the learning process proceeds to the next round. 

Approximating the posterior distribution using a Gaussian distribution is also called Laplace Thompson sampling~\citep{chapelle2011empirical}. Note that sampling from $\cN(\hbtheta_t,v_t \Vb_t^{-1})$ is usually implemented as $\tbtheta_t=\hbtheta_t+\sqrt{v_t}\Vb_t^{-1/2}\bzeta$ in practice, where $\bzeta$ is sampled from $\cN(\zero,\Ib)$ and $v_t>0$ is a scaling parameter. The computation complexity of calculating $\Vb^{-1/2}$ is at least $O(d^3)$ with Cholesky decomposition, which is prohibitively high, especially for high-dimensional machine learning problems. On the other hand, the Gaussian distribution used in Thompson sampling might not be a good approximation of the posterior distribution for general bandit models with more complicated structures than linear contextual bandits. 

\section{Langevin Monte Carlo Thompson Sampling }\label{sec:alg}

\begin{algorithm}
\caption{Langevin Monte Carlo Thompson Sampling ($\algname$)\label{alg:ts_lmc_general}}
\begin{algorithmic}[1]
\STATE Input:  step sizes $\{\eta_t>0\}_{t\geq1}$, inverse temperature parameters  $\{\beta_t\}_{t\geq1}$, loss function $L_t(\btheta)$, and reward model function $f(\xb,\btheta)$. $\btheta_{1,0}=\zero$, $K_0=0$.
\FOR{$t=1,2,\ldots$}
\STATE $\btheta_{t,0}=\btheta_{t-1,K_{t-1}}$
\FOR{$k=1,\ldots,K_t$}
\STATE sample a standard normal vector $\bepsilon_{t,k}\sim\cN(\zero,\Ib)$
\STATE $\btheta_{t,k}=\btheta_{t,k-1}-\eta_t\nabla L_t(\btheta_{t,k-1})+\sqrt{2\eta_t\beta_t^{-1}}\bepsilon_{t,k}$
\ENDFOR
\STATE Play arm $\xb_t=\argmax_{\xb\in\cX_t}f(\xb,\btheta_{t,K_t})$ and observe reward $r_t$
\ENDFOR
\end{algorithmic}
\end{algorithm}

In this paper, we propose the Langevin Monte Carlo Thompson Sampling ($\algname$) algorithm, which is presented in Algorithm~\ref{alg:ts_lmc_general}. Unlike existing work that use Laplace Approximation~\citep{chapelle2011empirical,agrawal2013thompson,kveton2020randomized,zhang2021neural}, which is essentially a Gaussian distribution, to approximate the unknown posterior distribution, we use Langevin Monte Carlo~\citep{roberts1996exponential,bakry2014analysis} to learn the exact posterior distribution of parameter $\btheta^*$ up to a high precision. One closely related work to ours is~\citet{pmlr-v119-mazumdar20a} which proposed to combine LMC and SGLD with Thompson sampling algorithms in finite-armed bandit problems without any contextual features. Their analysis heavily depends on their assumption on prior distributions and is hard to extend to contextual bandits (with potentially infinite arms) even for the simplest linear contextual bandits, which we will discuss in further details in the next section. 

In specific, Algorithm~\ref{alg:ts_lmc_general} works as follows. At the $t$-th round of the algorithm, we run the following subroutine for $K_t$ steps. For each $k=1,\ldots,K_t$, we have
\begin{align}\label{eq:def_lmc}
   \btheta_{t,k}=\btheta_{t,k-1}-\eta\nabla L_t(\btheta_{t,k-1})+\sqrt{2\eta_t\beta_t^{-1}}\bepsilon_{t,k},
\end{align}
where $\bepsilon_{t,k}$ is an isotropic Gaussian random vector in $\RR^d$,  $\eta>0$ is a step size parameter, $\beta_t$ is the inverse temperature parameter, and $L_t(\btheta)$ is loss function between the observed rewards $\{r_i\}_{i=1,\ldots,t-1}$ and estimated rewards $\{f(\xb,\btheta)\}$ that is specified by the user. \eqref{eq:def_lmc} is called the Langevin Monte Carlo (LMC) method in the approximate sampling literature~\citep{roberts1996exponential,bakry2014analysis,dalalyan2017theoretical,dalalyan2017further}, which could be viewed as the Euler-Maruyama discretization of the following stochastic differential equation from physics called Langevin dynamics~\citep{langevin1908theory}:
\begin{align}\label{eq:langevin_dynamics}
\dd  \btheta(s) = -\nabla L_t\big(\btheta(s)\big) \dd s + \sqrt{2\beta_t^{-1}}\dd \bB(s),
\end{align}
where $s>0$ is a continuous time index, $\beta>0$ is  the inverse temperature parameter and $\bB(t)\in\RR^d$ is a Brownian motion. 
It has been showed that under certain conditions on the drift term $-\nabla L(\btheta(t))$, Langevin dynamics will converge to a unique stationary distribution $\pi(\dd \xb)\propto e^{-\beta L(\xb)}\dd \xb$. Therefore, one  can use \eqref{alg:ts_lmc_general} to approximately sample from an arbitrary distribution $\pi_t\propto\exp(-\beta_t L_t(\btheta))$.

Note that Algorithm~\ref{alg:ts_lmc_general} is applicable to various different contextual bandit settings if we choose the corresponding reward model $f(\xb,\btheta)$ and log-density function $L_t(\btheta)$. Another advantage of our algorithm is that only noisy gradient descent update is performed in order to do proper exploration in different bandit problems. Thus our $\algname$ is both flexible in design and is easy to implement in practice. 
In the next few subsections, we show that how we can instantiate our algorithm for linear contextual bandits, generalized linear bandits, and neural contextual bandits respectively.
\subsection{Implication to Linear Contextual Bandits}
In linear contextual bandits, it is assumed that the reward generating function is $f(\xb)=\xb^{\top}\btheta^*$ for all $\xb\in\cX$.  Define the following loss function
\begin{align}\label{def:loss_time_t}
    L_t(\btheta)&=\sum_{i=1}^{t-1}\big(\xb_{i}^{\top}\btheta-r_i\big)^2+\lambda\|\btheta\|^2,
\end{align}
where $\lambda>0$ is a regularization parameter. Then we have the gradient of $L_t(\btheta)$ as $\nabla L_t(\btheta)=2(\Vb_t\btheta-\bbb_t)$, where $\Vb_t$ and $\bbb_t$ are defined in the same way as in \eqref{eq:V_b_update}. 

Based on the linear bandit model  and the loss function chosen in \eqref{def:loss_time_t}, we can show that the inner loop of Algorithm~\ref{alg:ts_lmc_general} generates samples approximately from the Gaussian posterior distribution.
\begin{proposition}\label{prop:equivalence_LMC_TS}
If the epoch length $K_t$ in Algorithm~\ref{alg:ts_lmc_general} is sufficiently large, the distribution of $K_t$ converges to  Gaussian distribution $\cN(\Vb_t^{-1}\bbb_t,\beta_{t}^{-1}\Vb_t^{-1})$ up to an arbitrary accuracy.
\end{proposition}
Note that LMC does not converge exactly to the posterior distribution but instead converges to it with an arbitrarily pre-chosen prevision parameter for large enough $K_t$. In the next section, we will show this will be sufficient for the proposed bandit algorithm to achieve a sublinear regret.

\begin{proof}
According to~\citet{roberts1996exponential,bakry2014analysis}, we know the Markov chain generated by Langevin dynamics \eqref{eq:langevin_dynamics} converges to a stationary distribution $\pi_t$, which is defined as $\pi_t(\btheta)=Z^{-1}\exp(-\beta_t L_t(\btheta))$, 
where $Z=\int\exp(-\beta_t L_t(\btheta))\dd\btheta$ is the normalization term. In the inner loop of Algorithm~\ref{alg:ts_lmc_general}, we apply the LMC update defined in \eqref{eq:def_lmc} which is a discretization of \eqref{eq:langevin_dynamics} and thus obtain another Markov chain $\{\btheta_{t,k}\}_{k=0,1,\ldots}$. A recent line of non-asymptotic analyses show that the Markov chain $\{\btheta_{t,k}\}_{k=0,1,\ldots}$ generated by LMC converges to $\pi_t$ up to an arbitrary accuracy for (strongly)-convex $L_t$ \citep{dalalyan2017theoretical} and nonconvex $L_t$ \citep{vempala2019rapid} respectively, as long as the epoch length $K_t$ of Algorithm~\ref{alg:ts_lmc_general} is large enough. 

In what follows, we show that $\pi_t$ is the Gaussian distribution we need in linear contextual bandits. By the definition in \eqref{def:loss_time_t}, we have
\begin{align*}
   L_t(\btheta)&=\sum_{i=1}^{t-1}\big(\xb_{i}^{\top}\btheta-r_i\big)^2+\lambda\|\btheta\|^2\\
   &=\sum_{i=1}^{t-1}\big(\btheta^{\top}\xb_{i}\xb_{i}^{\top}\btheta-\la\btheta,2r_i\xb_{i}\ra+r_i^2\big)+\lambda\|\btheta\|^2\\
   &=\btheta^{\top}\bigg[\lambda\Ib+\sum_{i=1}^{t-1}\xb_{i}\xb_{i}^{\top}\bigg]\btheta-2\bigg\la\btheta,\sum_{i=1}^{t-1}r_i\xb_{a_i}\bigg\ra+\sum_{i=1}^{t-1}r_i^2\\
   &=\btheta^{\top}\Vb_t\btheta-2\btheta^{\top}\bbb_t+\sum_{i=1}^{t-1}r_i^2,
\end{align*}
where the last equality is due to \eqref{eq:V_b_update}. 
We denote $\hbtheta_t=\argmin_{\btheta} L_t(\btheta)$ which is the solution of ridge regression~\citep{hoerl1970ridge}. It is easy to verify that the solution of the ridge regression problem has the form $\hbtheta_t=\Vb_t^{-1}\bbb_t$. Then, we have
\begin{align*}
    \big(\btheta-\hbtheta_t\big)^{\top}\Vb_t\big(\btheta-\hbtheta_t\big)&=\btheta^{\top}\Vb_t\btheta-2\btheta^{\top}\Vb_t\hbtheta_t+\hbtheta_t^{\top}\Vb_t\hbtheta_t\\
    &=\btheta^{\top}\Vb_t\btheta-2\btheta^{\top}\bbb_t+\hbtheta_t^{\top}\Vb_t\hbtheta_t,
\end{align*}
which immediately implies that
\begin{align*}
    \pi_t(\btheta)&\propto\exp(-\beta_t L_t(\btheta))\\
    &\propto\exp\big(-\beta_t\big(\btheta-\hbtheta_t\big)^{\top}\Vb_t\big(\btheta-\hbtheta_t\big)\big).
\end{align*}
Therefore, we can conclude that the  distribution of $\btheta_{t,K_t}$ converges to  Gaussian distribution $\cN(\hbtheta_t,\beta_t^{-1}\Vb_t^{-1})$. 
\end{proof}

Although the update in Algorithm~\ref{alg:ts_lmc_general} is presented as a full gradient descent step plus an isotropic noise, one can also replace the full gradient $\nabla L_t(\btheta_{t,k-1})$ with a stochastic gradient or a variance reduced stochastic gradient of the loss function $L_t(\btheta_{t,k-1})$ calculated from a mini-batch of data, which leads to Stochastic Gradient Langevin Dynamics (SGLD) algorithm~\citep{welling2011bayesian} and Stochastic Variance Reduced Gradient Langevin Dynamics (SVRG-LD) \citep{dubey2016variance,xu2018global}. And similar results to Proposition~\ref{prop:equivalence_LMC_TS} can be obtained by following the proofs in \citet{dalalyan2019user,xu2018global,zou2021faster}.

Now we have shown that our Algorithm~\ref{alg:ts_lmc_general} is similar to the Thompson sampling algorithm derived from Laplace approximation of the posterior distribution in linear contextual bandits. 
Nevertheless, our Algorithm~\ref{alg:ts_lmc_general} is more preferable than Thompson sampling in practice since at each iteration of $\algname$ we only need the computation of first order information, which is much more computationally efficient than computing the Cholesky decomposition for sampling from a general multivariate normal distribution in Thompson sampling when the feature dimension is high as we discussed in Section~\ref{sec:preliminary}.

\subsection{Implication to Generalized Linear Bandits}\label{sec:alg_glb}

In generalized linear bandits (GLB), the true reward $r$ for arm $\xb\in\cX_t$ at round $t$ is assumed to be from a generalized linear model (GLM)~\citep{mccullagh2019generalized}. Specifically, conditional on feature vector $\xb$, $r$ follows an exponential family distribution with mean $\mu(\xb^{\top}\btheta^*)$, where  
$\btheta^*$ is an unknown weight parameter that is shared across all arms, and $\mu(\cdot)$ is called the link function. It is worth noting that generalized linear bandits cover a class of common bandit models used in practice. For instance, when $\mu(z)=z$ is the identity function, it reduces to linear contextual bandits; when $\mu(z)=1/(1+e^{-z})$ is the sigmoid function, it reduces to the logistic bandits~\citep{dong2019performance}.  

Based on samples $\{\xb_1,r_1,\ldots,\xb_{t-1},r_{t-1}\}$, 
the negative log-likelihood function is defined as $\tilde L_t(\btheta)=\sum_{i=1}^{t-1}(m(\xb_{a_i}^{\top}\btheta)-r_i\xb_{a_i}^{\top}\btheta)$, 
where $m(z)$ is twice differentiable and $m'(z)=\mu(z)$ is the link function defined in the previous paragraph.
Existing Thompson sampling based algorithms on generalized linear bandits~\citep{filippi2010parametric,kveton2020randomized,ding2021efficient} usually first solve the following MLE estimator 
\begin{align}\label{eq:MLE_glb}
    \hat\btheta_t=\argmax_{\btheta}  \sum_{i=1}^{t-1}(r_i\xb_{a_i}^{\top}\btheta-m(\xb_{a_i}^{\top}\btheta)),
\end{align}
and then construct a Laplace approximation of the underlying posterior distribution, which is given by Gaussian distribution $\cN(\hbtheta_t,a^2(\nabla^2 \tilde L_t(\hbtheta_t))^{-1})$, where $a>0$ is a scaling parameter. Similar to Thompson sampling for linear contextual bandits discussed in Section~\ref{sec:preliminary}, sampling from $\cN(\hbtheta_t,a^2(\nabla^2 \tilde L_t(\hbtheta_t))^{-1})$ is computationally inefficient in high dimensional applications. Moreover, due to the existence of the link function, the posterior itself is not necessary a Gaussian distribution, and thus the Laplace approximation might cause a fixed approximation error.

In contrast, our $\algname$ algorithm can be easily applied to GLB by choosing the reward model as $f(\xb,\btheta^*)=\mu(\xb^{\top}\btheta^*)$ and the loss function $L_t(\btheta)$ as the following regularized negative log-likelihood function
\begin{align}\label{def:loss_glb_nll}
    L_t(\btheta)=\sum_{i=1}^{t-1}(m(\xb_{a_i}^{\top}\btheta)-r_i\xb_{a_i}^{\top}\btheta)+\lambda\|\btheta\|_2^2,
\end{align}
where $\lambda>0$ is a tuning parameter. Under some standard conditions on the link function in GLB, it could be shown that the posterior density $\pi_t\propto\exp(-\beta_t L_t(\btheta))$ is strongly log-concave and log-smooth. Similar to the proof of Proposition~\ref{prop:equivalence_LMC_TS}, by recalling results in the study of Langevin Monte Carlo~\citep{dalalyan2017theoretical}, we can show that the distribution of iterates $\btheta_{t,K_t}$ in Algorithm~\ref{alg:ts_lmc_general} converges to the true posterior distribution $\pi_t$ if the epoch length $K_t$ of the inner loop is sufficiently large.
In addition, due to the simplicity of our algorithm, we only need to perform gradient descent based updates, which is computationally more efficient than Laplace approximation based Thompson sampling for generalized linear bandits~\citep{kveton2020randomized}.

\subsection{Implication to Neural Contextual Bandits}
Our algorithm is also applicable to more general contextual bandit problems, where the reward function $f(\xb,\btheta^*)$ is a neural network with $\xb$ as its input and $\btheta^*$ as the collection of all weight matrices. This type of bandit model is usually referred to as the neural contextual bandits  in the literature~\citep{riquelme2018deep,zhou2020neural,zhang2021neural,xu2022neural}. One possible choice of the  function $L_t(\btheta)$ used in Algorithm~\ref{alg:ts_lmc_general} is  the squared loss: 
\begin{align}
    L_t(\btheta)&=\sum_{i=1}^{t-1}\big(f(\xb_{i},\btheta)-r_i\big)^2+\lambda\|\btheta\|^2,
\end{align}
where $\lambda>0$. Due to the flexibility in the choice of loss functions in our method, one can always choose another loss function $L_t(\btheta)$ based on the belief in the prior and posterior distributions in specific applications to boost the empirical performance. This makes our method directly applicable to complicated deep learning applications. 

\section{Theoretical Analysis of $\algname$ for Linear Contextual Bandits}

In this section, we provide the regret analysis of our proposed $\algname$ algorithm when we apply it to a specific contextual bandit problem, viz., the linear contextual bandit. We first state the assumption on the details of the model.
\begin{assumption}\label{assump:linear_bandit_sugaussian}
There is an unknown parameter $\btheta^*\in\RR^d$ such that for any arm $\xb\in\cX\subseteq\RR^d$, the reward is $r(\xb)=\xb^{\top}\btheta^*+\xi$, where $\xi$ is assumed to be a $R$-subGaussian random variable for some constant $R>0$. 
\end{assumption}

The following theorem states the regret bound of $\algname$. 
\begin{theorem}\label{thm:regret_lcb}
Under Assumption~\ref{assump:linear_bandit_sugaussian}, we choose a linear reward model $f(\xb,\btheta)=\xb^{\top}\btheta$ in Algorithm~\ref{alg:ts_lmc_general}. Let $\delta\in(0,1)$. For any $j=1,2,\ldots$, let the step size $\eta_j=1/(4\lambda_{\max}(\Vb_t))$, the epoch length $K_j=\kappa_j\log(3R\sqrt{2dT\log(T^3/\delta)})$, and the inverse temperature $\beta_j^{-1}=4(R\sqrt{d\log(T^3/\delta)})$, where $\kappa_j=\lambda_{\max}(\Vb_j)/\lambda_{\min}(\Vb_j)$ is the condition number of $\Vb_j$. Then with probability $1-\delta$, it holds that
\begin{align*}
    R(T)\leq CRd\log(1/\delta)\sqrt{dT\log^3(1+T/(\lambda d))},
\end{align*}
where $C>0$ is an absolute constant that is independent of the problem.
\end{theorem}

Note that the regret upper bound in Theorem~\ref{thm:regret_lcb} is in the order of $\widetilde O(d\sqrt{dT})$ which matches the best result for Thompson sampling based algorithm in linear contextual bandits with infinite arms~\citep{agrawal2013thompson}. Moreover, if the arm set $|\cX_t|\leq N$ is finite in each round for some integer $N$, following a similar proof as in~\citet{agrawal2013thompson}, this regret bound can be improved to $\tO(d\sqrt{T})$, where a logarithmic dependence on the number of arms $N$ replaces the additional term $O(\sqrt{d})$ and is omitted in the $\tO(\cdot)$ notation. This shows that $\algname$ is theoretically comparable to Laplace approximation based Thompson sampling in linear contextual bandits. Nevertheless, due to the fact that we add multiple noises in Algorithm~\ref{alg:ts_lmc_general} at different time steps, the coupling of these noises makes the regret analysis of existing Laplace approximation Thompson sampling not directly applicable to our case. In specific, it has been shown by~\citet{phan2019thompson} that the approximation error caused by the posterior sampling step for Thompson sampling can yield a linear regret in general and thus we have to develop nontrivial proof techniques to achieve a sublinear regret for LMC-TS.

\begin{remark}
The only existing work that study the sublinear regret of TS with approximate sampling is given by~\citet{pmlr-v119-mazumdar20a}.
Compared with their work which combines Langevin algorithms with Thompson sampling for multi-armed bandits, our analysis applies to a more general class of bandit problems which covers MAB as a special case. Moreover, we do not assume that the reward distribution of a single data point is strongly log-concave, which is hard to be justified in contextual bandits. Specifically, they assume that the reward $r$ conditional on context $\xb_i$ has a distribution $p(r|\xb,\btheta^*)$ which is strongly log-concave w.r.t. both $\xb$ and $\btheta^*$. However, this assumption is not satisfied even for Gaussian rewards. For instance, if $r$ is a Gaussian reward with mean $\xb^{\top}\btheta^*$ and unit variance, then $-\log p(r|\xb,\btheta^*)$ is not strongly convex with respect to $\xb\in\RR^d$ when the dimension $d$ is larger than $1$.  In contrast, we only assume that the reward is subGaussian, which is among the most common assumptions in linear contextual bandits~\citep{abbasi2011improved,agrawal2013thompson}. 
\end{remark}

\begin{remark}
We note that the epoch length of the inner loop of Algorithm~\ref{alg:ts_lmc_general} depends on the condition number $\kappa_j=\lambda_{\max}(\Vb_j)/\lambda_{\min}(\Vb_j)$ which could be $O(j)$ in the worst case. This is due to that the optimization loss function $L_t(\btheta)$ is the sum of $t$ squared loss functions and we do not assume each loss function is strongly convex. 
Moreover, under certain assumption on the diversity of the arm set as is studied in
\citet{hamidi2020worst,wu2020stochastic}, the condition number will become $O(1)$. On the other hand, if we apply Newton's method to minimize the loss function $L_t(\btheta)$ in each round, we could get rid of the condition number $\kappa_j$ in the dependence of the epoch length of the inner loop $K_j$. Lastly, we observe from our empirical study that $\algname$ often requires a small number of iterations to achieve a good performance.
\end{remark}

\section{Empirical Evaluation of $\algname$}\label{sec:experiment}

In this section, we conduct experiments on both synthetic datasets and real-world datasets to show that the proposed algorithm achieves the best performance in terms of regret minimization and also is scalable to large-scale and high-dimensional problems. 
All experiments are conducted on Amazon EC2 P3 instances with NVIDIA V100 GPUs and Broadwell E5-2686 v4 processors. 
Our implementation can be found at \url{https://github.com/devzhk/LMCTS}.

\textbf{Benchmarks and baseline algorithms} As we discussed in Section~\ref{sec:alg}, our $\algname$ algorithm is applicable to many different contextual bandit problems. Hence we compare it with baseline algorithms in different bandit settings including linear contextual bandits, logistic bandits, quadratic bandits, and neural bandits (also known as deep bandits) respectively. 
For linear bandit problems, we compare our algorithm with baseline linear bandit algorithms such as LinUCB~\citep{chu2011contextual}, LinTS~\citep{agrawal2013thompson}, and the $\epsilon$-greedy algorithm. For logistic bandit problems, we compare our algorithm with existing state-of-the-art algorithms for generalized linear bandits including
UCB-GLM~\citep{li2017provably}, GLM-TSL~\citep{kveton2020randomized}, SGD-TS~\citep{ding2021efficient}, and $\epsilon$-greedy. For quadratic bandits and neural bandits, we compare our algorithm additionally with NeuralUCB~\citep{zhou2020neural}, NeuralTS~\citep{zhang2021neural}, Neural-LinUCB~\citep{xu2022neural}, and Neural $\epsilon$-greedy~\citep{riquelme2018deep} which applies $\epsilon$-greedy exploration to a neural network reward model trained by SGD.

\subsection{Simulation Study on Linear, Logistic, and Quadratic Contextual Bandits}\label{sec:simulation}
We first compare our algorithm with baseline methods on simulated bandit problems presented in Section~\ref{sec:preliminary} including linear bandits, logistic bandits, and quadratic bandits, where the true reward model $f(\xb,\btheta^*)$ is known but the weight parameter $\btheta^*\in\RR^d$ is unknown. Throughout this subsection, the context feature dimension is  $d=20$, the size of the arm set at round $t$ is $|\cX_t|=50$, and the time horizon of all learning algorithms is $T=10000$. 

\textbf{Linear contextual bandits:} We first generate a linear contextual dataset following the problem setup in Section~\ref{sec:preliminary}. To simulate the bandit environment, we first generate $\btheta^*\in\mathbb{R}^d$ with each coordinate randomly sampled from $\cN(0,1)$ and then scale $\btheta^*$ to unit norm. We consider the following two settings: (1) we have a fixed arm set $\cX\subseteq\RR^d$ that remains the same during the whole learning process; (2) we receive a new arm set $\cX_t\subseteq\RR^d$ at each round $t$. For the feature vectors, we generate $\xb\in\mathbb{R}^{d}$ with each coordinate randomly sampled from $\cN(0,1)$ and then scale each vector to unit norm. The true reward for an arm $\xb$ is then generated by $r(\xb)=\btheta^{*\top} \xb+\xi$, where the noise $\xi$ is sampled from $\cN(0,\sigma^2)$ with $\sigma^2=0.5$.

\textbf{Logistic bandits:} In the logistic bandit experiment, we follow the setting in~\citet{kveton2020randomized} and consider the fixed arm setting where $\cX \subseteq \RR^d$ with the context dimension $d=20$ and the size of arm set $|\cX|=50$. Each contextual vector is randomly generated from $\mathcal{N}(\zero, \Ib)$ and scaled to unit norm. The reward for arm $\xb\in\cX$ is generated by a Bernoulli distribution, viz., $r(\xb) \sim \mathrm{Ber}\left(\mu(\btheta^{*\top}\xb)\right)$, where $\btheta^*\in \mathbb{R}^d$ is sampled from $\mathcal{N}(\zero, \Ib)$ and scaled to unit norm, and $\mu(v)=1/(1+\exp(-v))$ is the logistic function.

\textbf{Quadratic bandits:} Following the setting in~\citet{zhou2020neural}, we generate a quadratic contextual bandit problem. We generate a changing  arm set $\cX_t\subseteq\RR^d$ at each round. The context dimension $d=20$. The size of arm set $|\cX_t|=50$ at each round. Each contextual vector $\xb\in\cX_t$ is randomly generated from $\mathcal{N}(\zero, \Ib)$ and scaled to unit norm. At round $t$, the reward function for the chosen arm $\xb_t\in\cX_t$ is given by 
$r_t = 10\left(\btheta^{*\top} \xb_t\right)^2 + \xi$, where $\btheta^*\in \mathbb{R}^d$ is sampled from $\mathcal{N}(\zero, \Ib)$ and scaled to unit norm, and the noise $\xi \sim \mathcal{N}(0, 1)$. 

\textbf{Implementation details:}  
For linear bandits, a linear reward model $f(\xb,\btheta)=\xb^{\top}\btheta$ is used in all algorithms. For LinUCB, we set the UCB bonus parameter as $\nu_t=c\sqrt{d\log t}$ following~\citet{li2010contextual} and find the best parameter $c$ by performing a grid search. For LinTS, we set the variance parameter as $\nu=c\sqrt{d\log T}$  following the theory in~\citet{agrawal2013thompson}, and pick the best hyperparameter constant $c$ from a grid search. For $\epsilon$-greedy, the exploration rate is $\frac{c}{\sqrt{t}}$ at time $t$, where $c$ is selected by a gird search. For $\algname$, we set the step size $\eta_t = \frac{\eta_0}{t}$ as suggested in our theory and do a grid search for the constant $\eta_0$ and the temperature parameter $\beta^{-1}$. We fix the epoch length for the inner loop of our algorithm as $K_t=100$ for all $t$. 

For logistic bandits, a generalized linear reward model $f(\xb,\btheta)=\mu(\xb^{\top}\btheta)$ is used, where $\mu(v)=1/(1+\exp(-v))$. The MLE estimator in \eqref{eq:MLE_glb} is solved via SGD for UCB-GLM, GLM-TSL, and $\epsilon$-greedy. For UCB-GLM, and GLM-TSL, we follow the same parameter setting proposed in~\citet{kveton2020randomized}. For $\algname$, we use the loss function defined in \eqref{def:loss_glb_nll}, and the step size and temperature parameters are tuned in the same way as in linear bandits.

For quadratic bandits, a 4-layer fully-connected neural network with width 20 is used in NeuralUCB, NeuralTS,  Neural-LinUCB, Neural $\epsilon$-greedy, and $\algname$. We try both ReLU and LeakyReLU and pick the best activation function for each algorithm. Neural networks are all updated by $100$ gradient descent steps every round. 
Following the original implementation of NeuralTS and NeuralUCB, the matrix inverse is approximated by the inverse of diagonal. We perform grid search for the regularization parameter $\lambda$ and variance parameter $\nu$ in their work. 

To make a fair comparison, we first perform grid searches for the parameters of all algorithms. We then fix the best hyperparameters, and shuffle the order of the dataset and repeat experiments for $10$ times with different random seeds. 

\begin{figure*}[!htbp]
    \centering
    \subfigure[Linear bandit (fixed arm set)]{\includegraphics[scale=0.3]{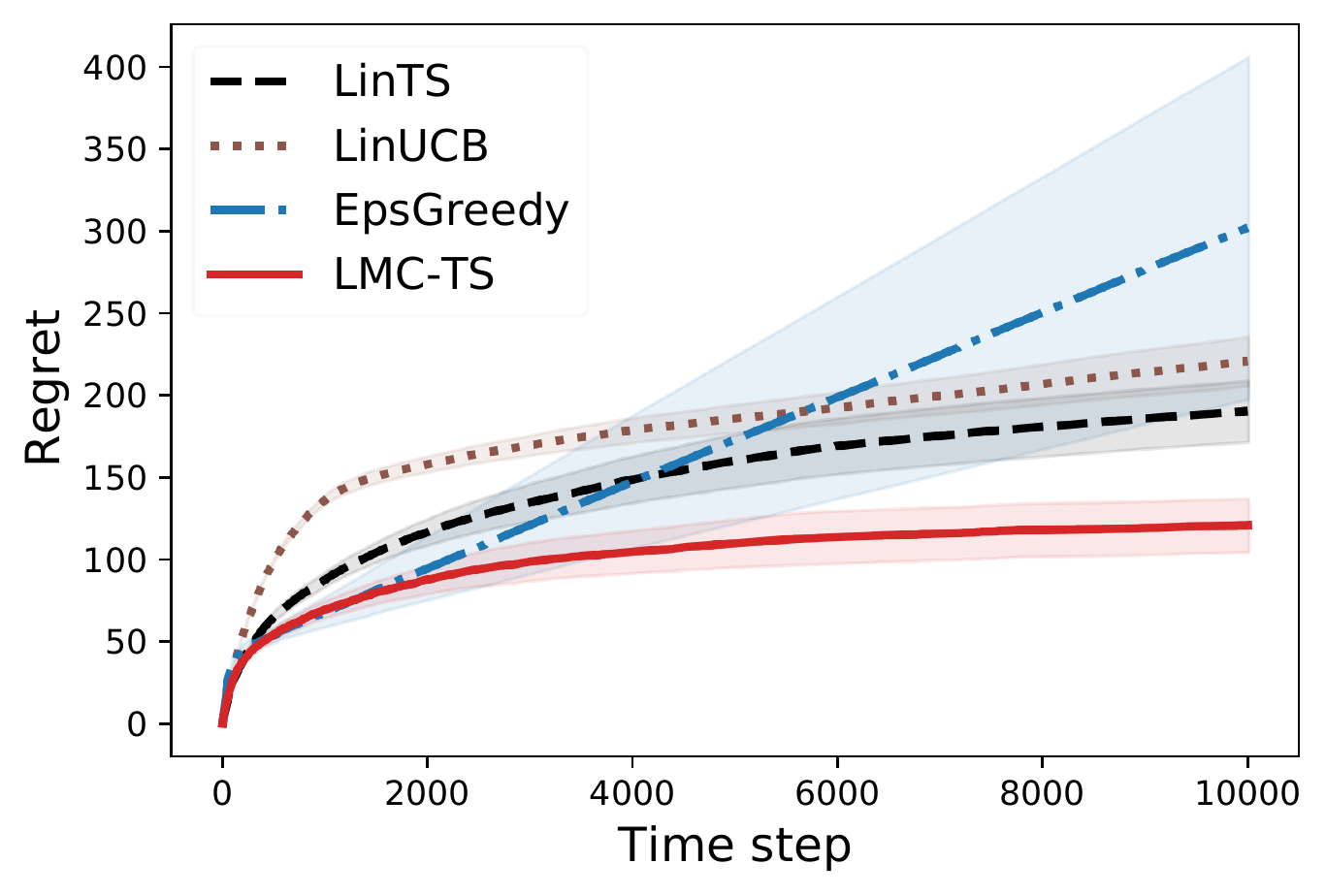}\label{fig:simu_linear_fix}}
    \subfigure[Linear bandit]{\includegraphics[scale=0.3]{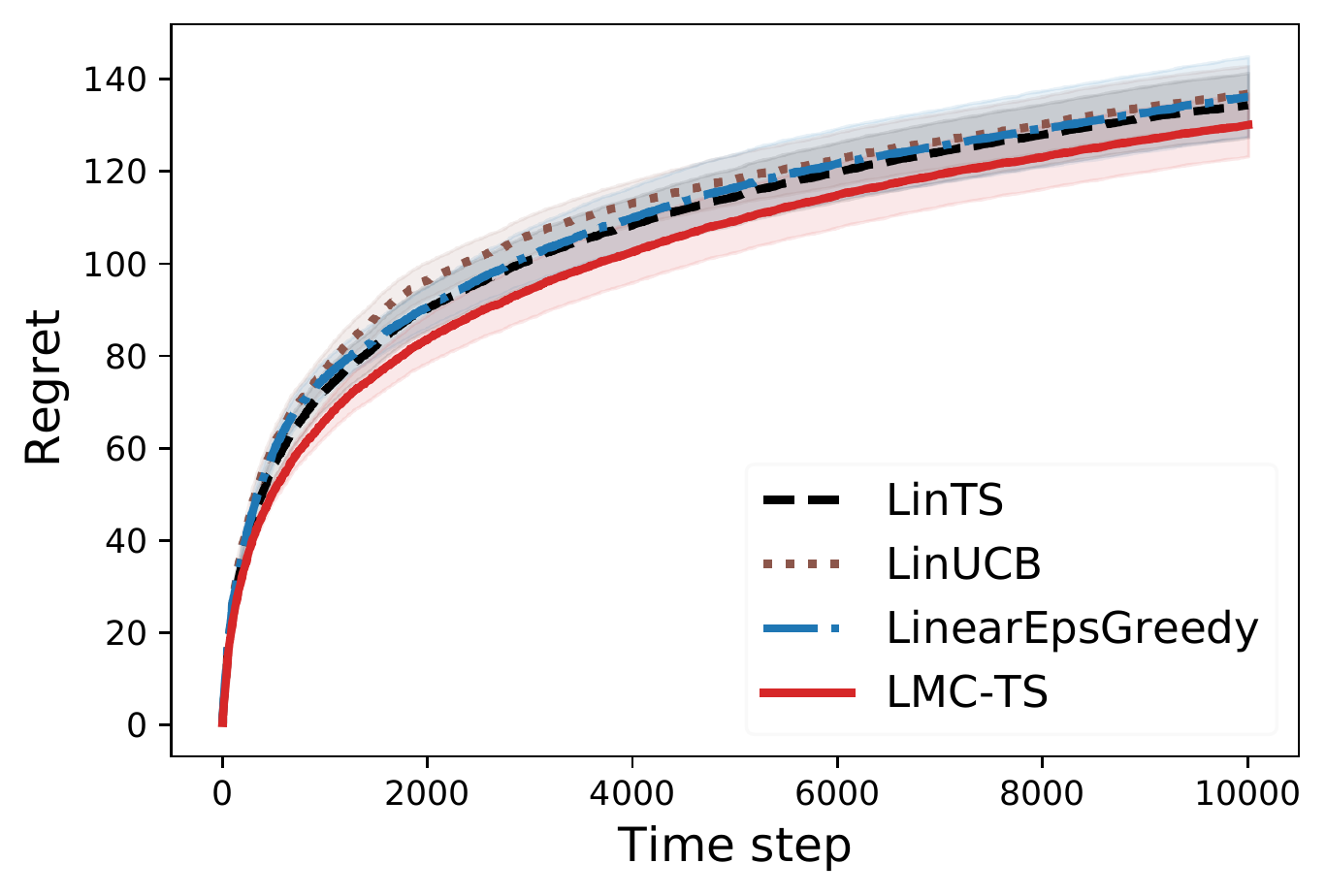}\label{fig:simu_linear_changing}}   
    \subfigure[Logistic bandit]{\includegraphics[scale=0.3]{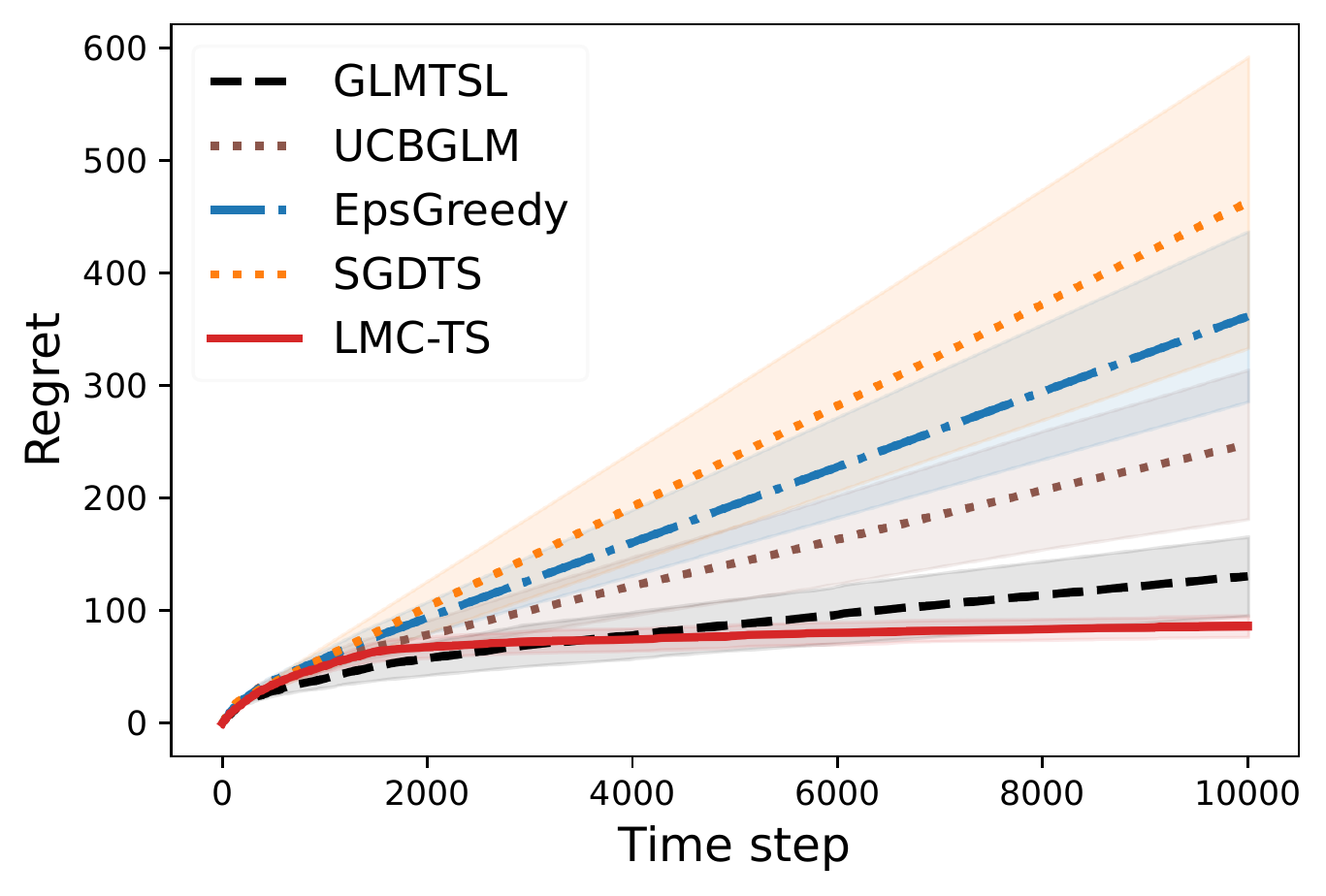}\label{fig:simu_logistic}}
    \subfigure[Quadratic bandit]{\includegraphics[scale=0.3]{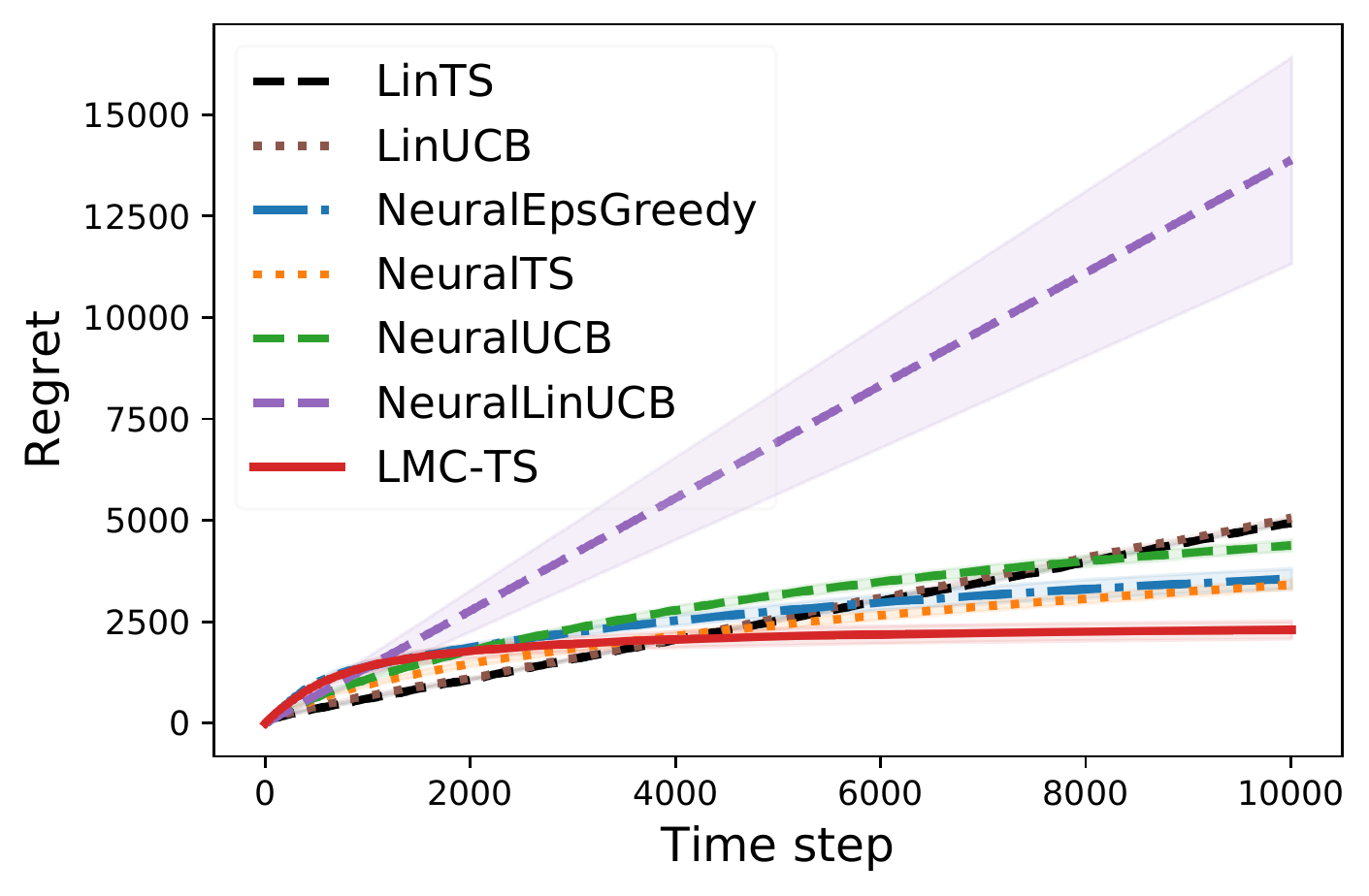}\label{fig:simu_quadratic}}    
    \caption{Regret comparison on simulated bandit problems. The mean and standard error are reported over 10 runs.
    }
    \label{fig:simulation}
\end{figure*}

\textbf{Results:} 
We report the mean and the standard error of the accumulative regret of different algorithms over $10$ runs on all simulated bandit problems in Figure~\ref{fig:simulation}. The results on linear contextual bandits are shown in Figure~\ref{fig:simu_linear_fix} with a fixed arm set and in Figure~\ref{fig:simu_linear_changing} with time-changing arm sets. Our method $\algname$ achieves the best performance in both settings, and the performance gain in the changing arm setting is slightly lower since it is a more challenging problem. The results for logistic bandits are shown in Figure~\ref{fig:simu_logistic}, and $\algname$ again outperforms baseline methods. The results for quadratic bandits are shown in Figure~\ref{fig:simu_quadratic}. In this setting, LinUCB and LinTS works poorly due to their dependence on the linear bandit structure. All the other algorithms use a neural network to model the reward, and our method achieves significant lower regret. Neural-LinUCB performs much worse than other baseline methods, possibly due to its insufficient exploration in this setting.

\begin{table*}[!ht]
    \centering
    \caption{Specifications of real-world datasets used in this paper.     \label{table:dataset}}
    \begin{sc}
    \begin{tabular}{lcccccc}
    \toprule
    &\emph{ Shuttle} & \emph{MagicTelescope} &\emph{Mushroom} &\emph{Covertype}&\emph{CIFAR10}\\
    \midrule
    Number of  attributes       & 9     & 10    & 22    & 54    & $3\times32\times32$\\
    Number of arms              & 7     & 2     & 2     & 7     & 10\\
    Dimension of context feature & 63 & 20& 48& 378  & 30720\\
    Number of instances         & 58,000& 19,020& 8124  &  581,012&  10,000\\ 
    \bottomrule
    \end{tabular}
    \end{sc}
\end{table*}
\begin{figure*}[!htbp]
    \centering
    \subfigure[Shuttle]{\includegraphics[scale=0.3]{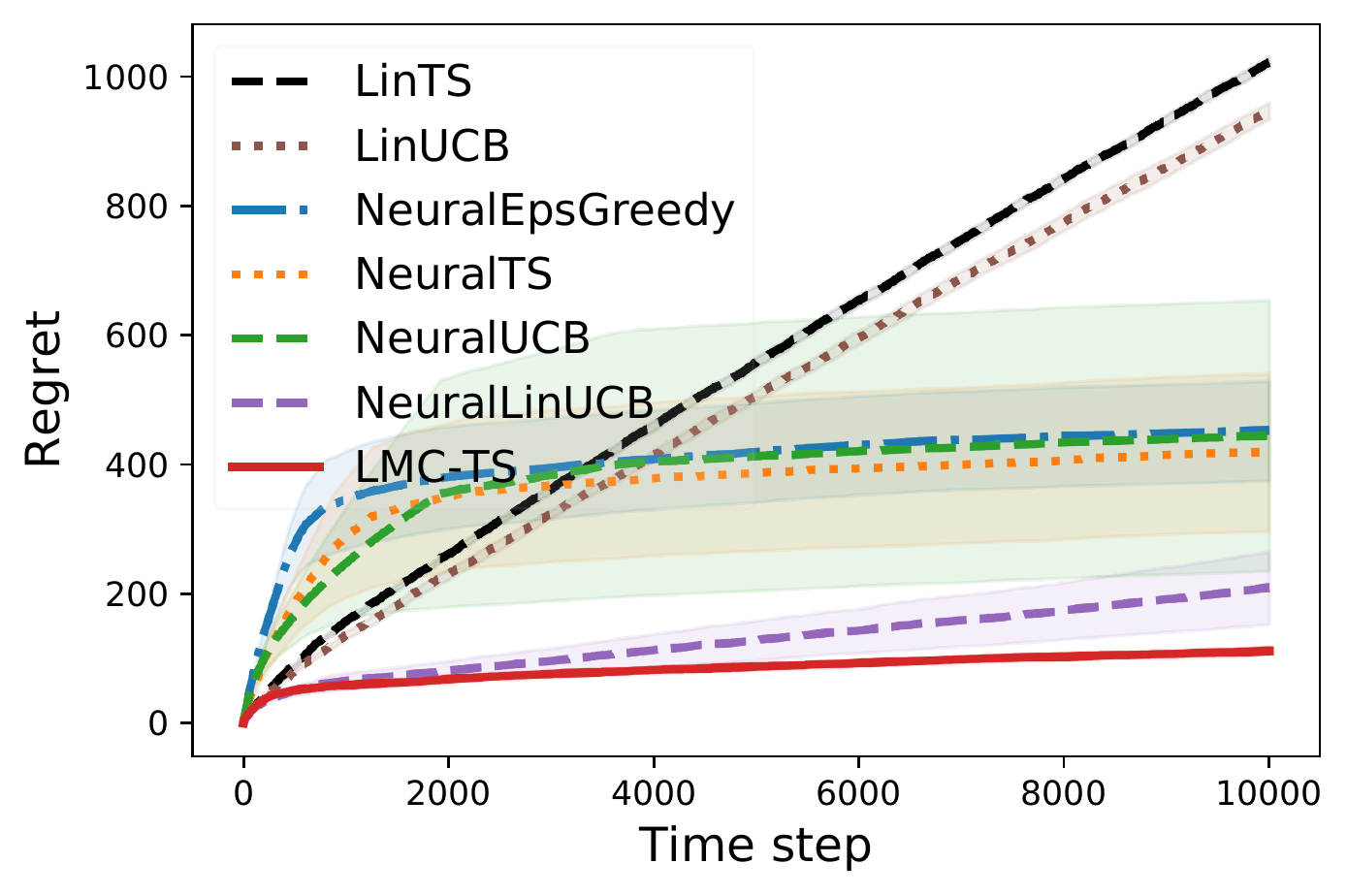}}
    \subfigure[MagicTelescope]{\includegraphics[scale=0.3]{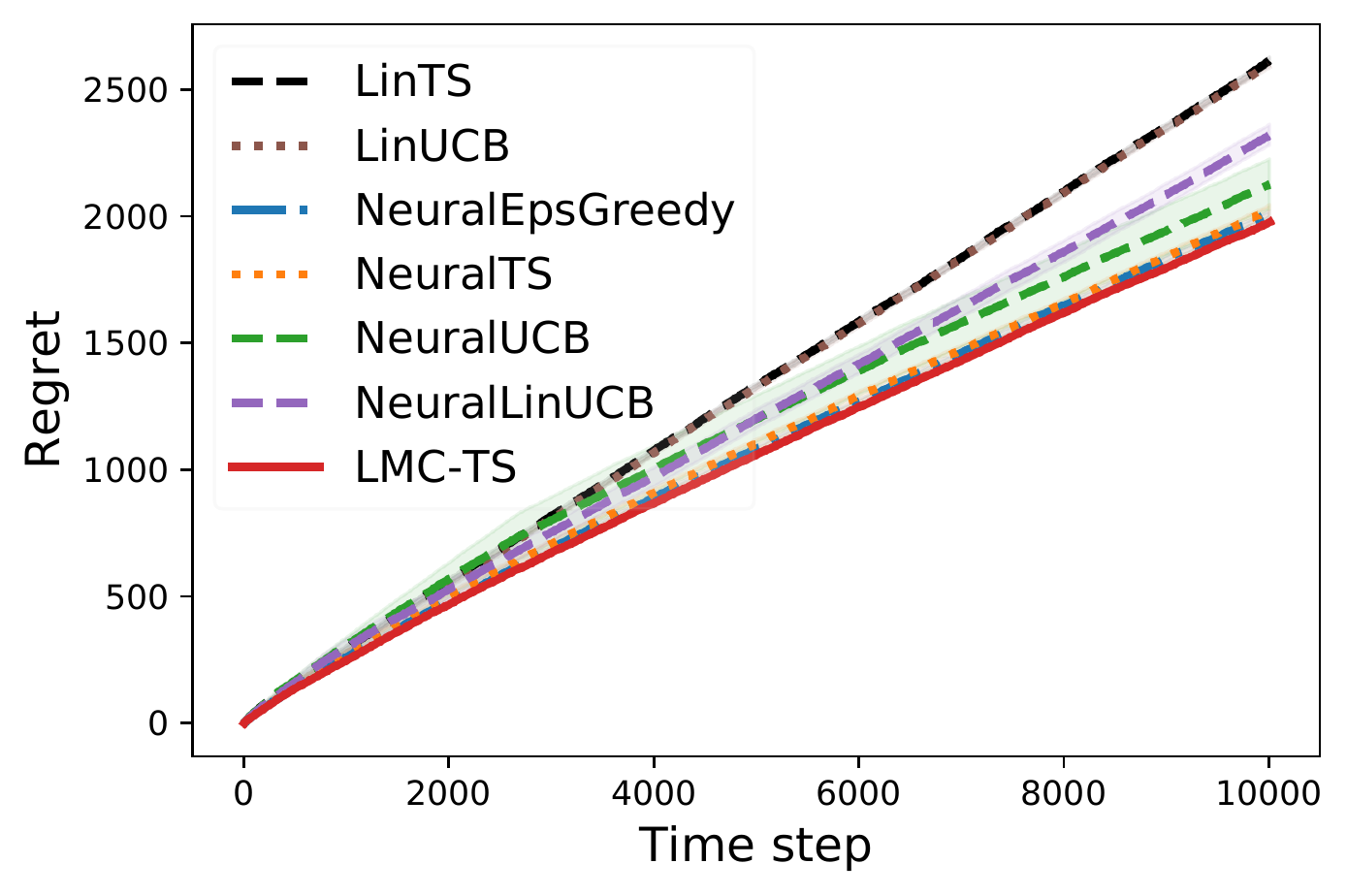}}    
    \subfigure[Mushroom]{\includegraphics[scale=0.3]{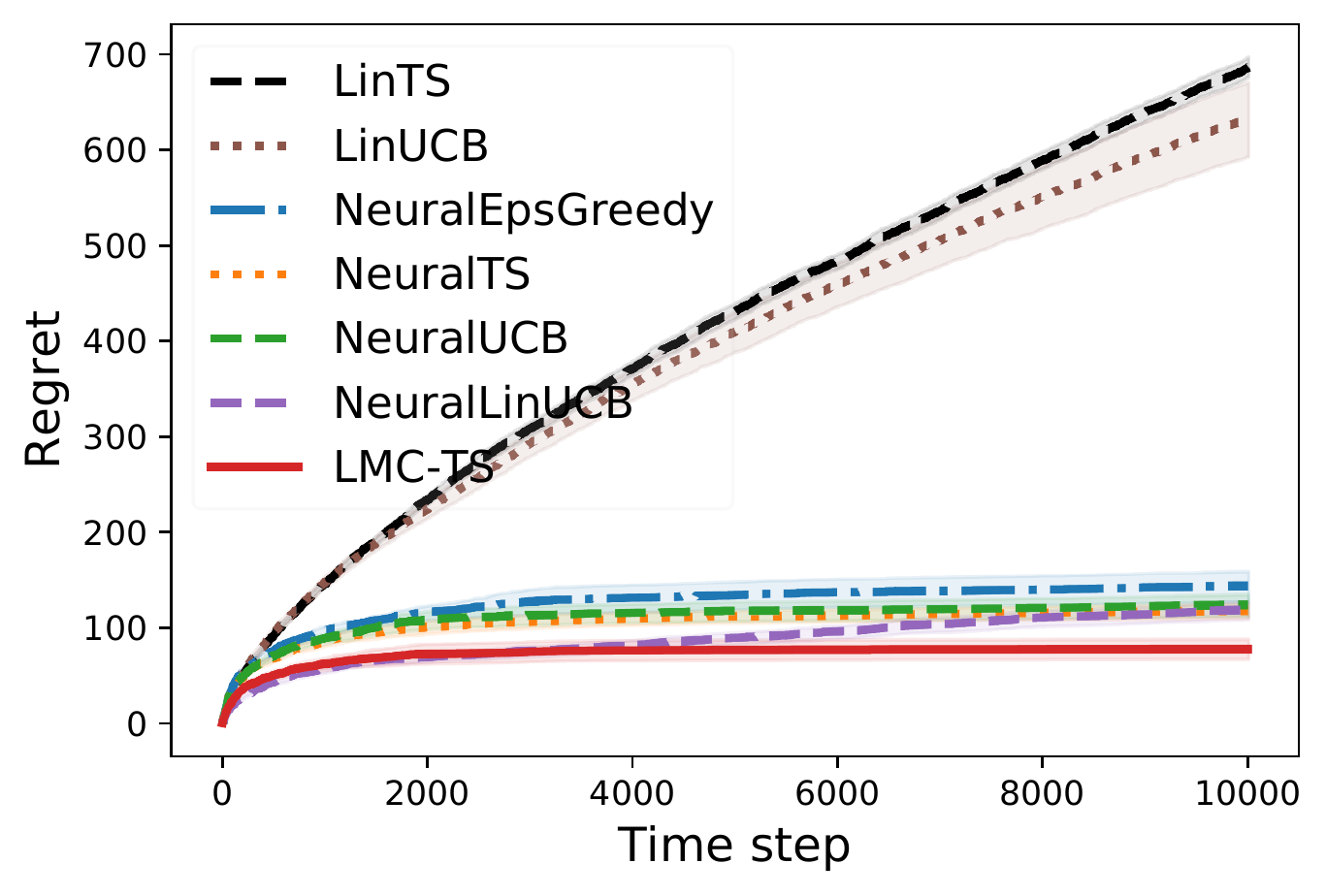}}       
    \subfigure[Covertype]{\includegraphics[scale=0.3]{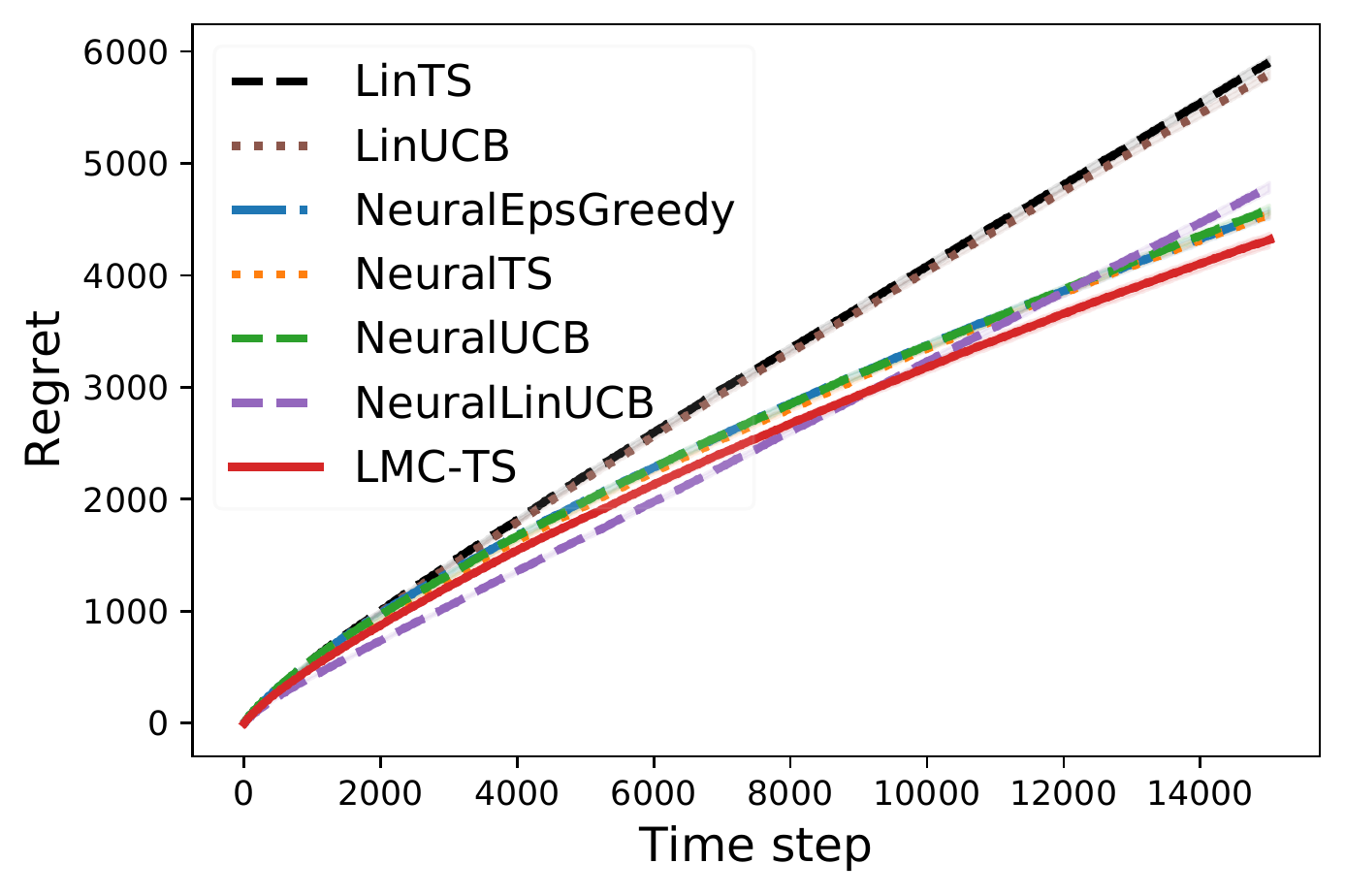}}
    \caption{Regret comparison on UCI datasets. The mean and standard error are reported over 10 runs. \label{fig:regret_uci}}
\end{figure*}

\begin{figure}[htbp]
    \centering
    \subfigure[Cumulative regret]{\includegraphics[scale=0.28]{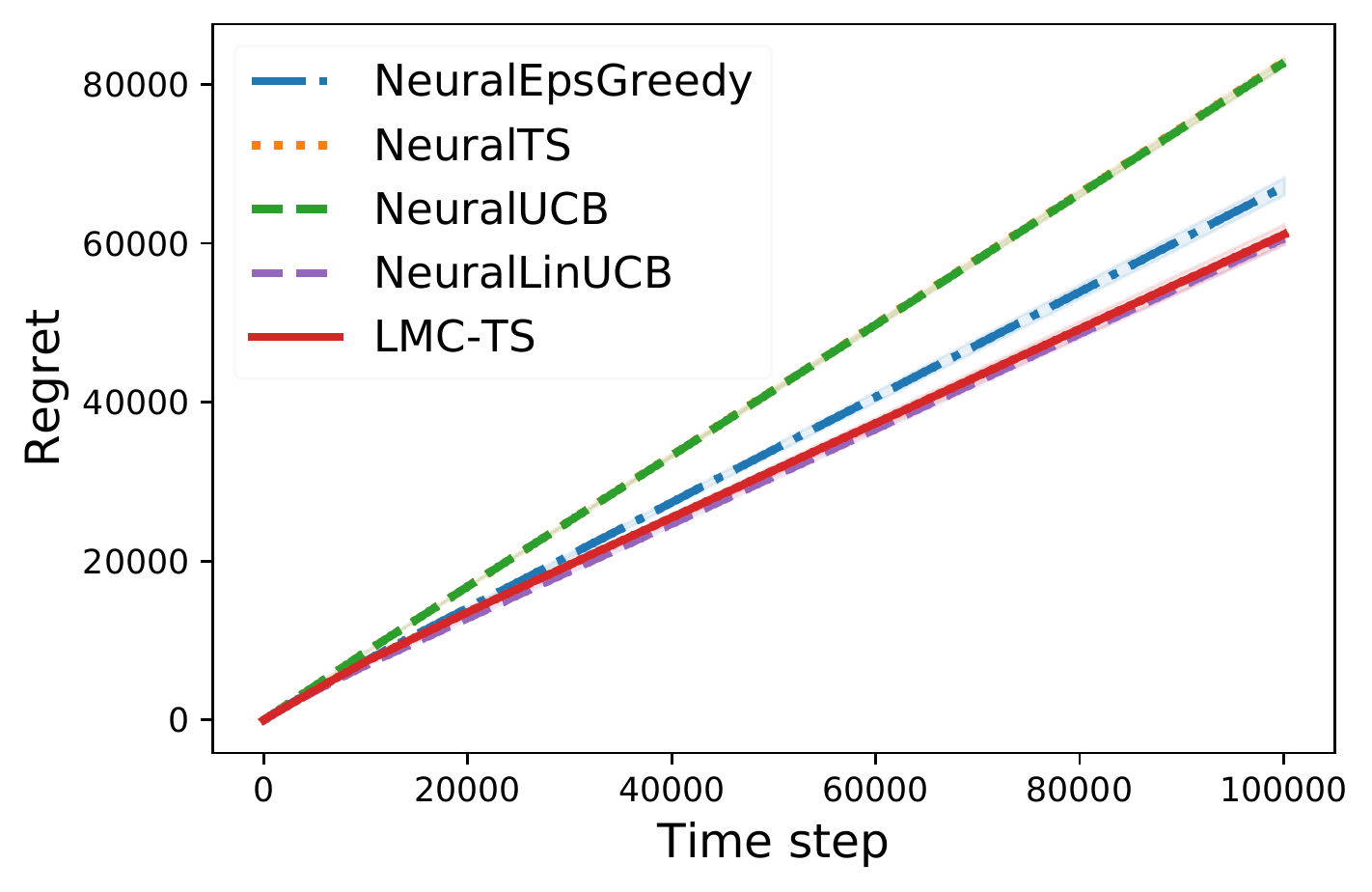}\label{fig:regret_cifar_cumulative}}    
    \subfigure[Averaged regret]{\includegraphics[scale=0.28]{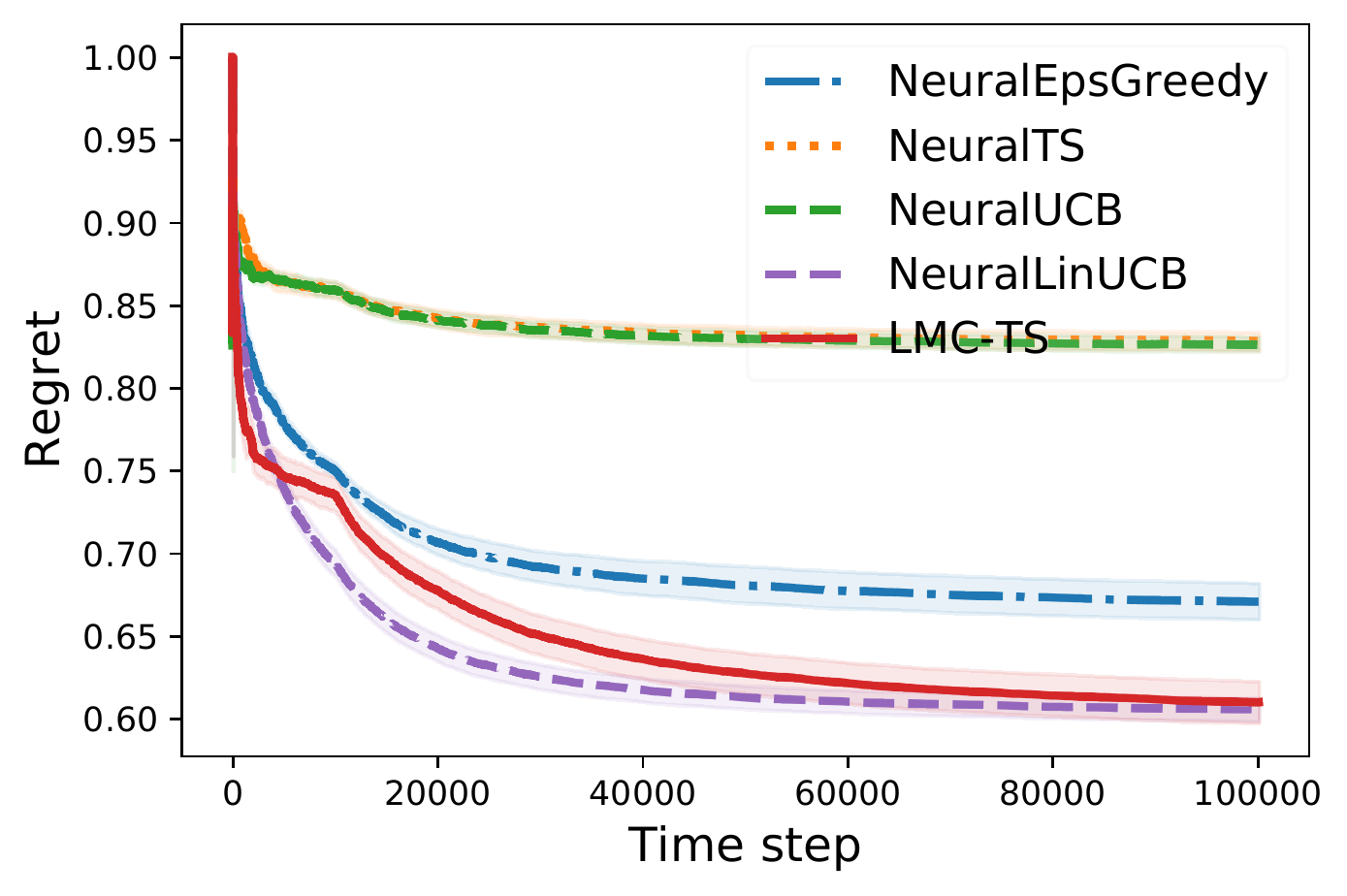}\label{fig:regret_cifar_average}}   
    \caption{Regret comparison on CIFAR10. The mean and standard error are reported over 5 runs. \label{fig:regret_cifar}}
\end{figure}
\subsection{Real-World Datasets}

In this subsection, we consider UCI machine learning datasets~\citep{dua2019uci} including \emph{Shuttle}, \emph{MagicTelescope}, \emph{Mushroom}, and \emph{Covertype},  and a high dimension image dataset \emph{CIFAR10}~\citep{krizhevsky2009learning}. 
The specifications of these datasets are summarized in Table~\ref{table:dataset}. 
To use these $N$-class classification datasets for contextual bandit problems, we follow~\citet{riquelme2018deep,kveton2020randomized} to construct context vectors for different arms in the following way: given a data feature $\xb\in \RR^d$, we transform it into $N$ contextual vectors $\xb^{(1)}=(\xb,\zero,\dots,\zero)$, $\ldots$, $\xb^{(N)}=(\zero,\ldots,\zero,\xb)\in \RR^{Nd}$. Only the arm $\xb^{(j)}$ where $j$ matches the correct class of this data has reward 1 and all other arms have reward 0. 

\textbf{Implementation details:} For all methods using neural networks, we choose the best architecture between a two-layer neural network with width 100 and a four-layer neural network with width 50. Activation function is also selected by grid search over ReLU and LeakyReLU. All baseline algorithms update the neural network by 100 gradient descent steps every round. The parameters of all methods are tuned in the same way as in Section~\ref{sec:simulation} for quadratic bandits.

\textbf{Results:} We plot the average regret with std over $10$ repeats for all algorithms on UCI datasets in Figure~\ref{fig:regret_uci}. It can be seen that our algorithm $\algname$ consistently outperforms other neural network based algorithms. The results for linear bandit based approaches are the worst on all datasets. The results on CIFAR10 is displayed in Figure~\ref{fig:regret_cifar}. Note that the dimension of this dataset is much higher than those of UCI datasets that are used in existing neural contextual bandit papers~\citep{riquelme2018deep,zhou2020neural,zhang2021neural,xu2022neural}. We plot the cumulative regret in Figure~\ref{fig:regret_cifar_cumulative} and the average regret in Figure~\ref{fig:regret_cifar_average}, which shows that $\algname$ achieves a sublinear cumulative regret and significantly outperforms all baseline methods except Neural-LinUCB which converges almost as fast as LMC-TS. This is due to the shallow exploration of Neural-LinUCB which only perform UCB exploration on the last layer parameter of the neural network. However, it also leads to a higher asymptotic regret than LMC-TS accoding to Figure~\ref{fig:regret_cifar_average}, which explores over the whole parameter space. This shows the potential of our algorithm in high dimensional bandit problems that use deep neural networks. 

We also compare the runtime of our algorithm with NeuralTS, which is a Laplace approximation based Thompson sampling algorithm, to demonstrate the computational efficiency of $\algname$.
The runtime for running $1000$ rounds is shown in Table~\ref{table:profiling_neural}. We present both the total runtime of these rounds and the time only for arm selection. It can be seen that $\algname$ is  more computationally efficient than NeuralTS, which is due to the fact that we do not need to calculate the inverse of a high dimensional matrix or sample from a nonisotropic Gaussian distribution. 

\begin{table}[h]
{\small
	\centering
	\caption{Runtime (seconds) on the CIFAR10 dataset for running 1000 rounds. 
	\label{table:profiling_neural}}
    \begin{sc}
    \begin{tabular}{cccc}
    \toprule
    & Time for arm selection & Total time\\
    \midrule
    NeuralTS  & 174.1& 776.0\\
    LMC-TS & 9.0& 662.7\\
    \bottomrule
    \end{tabular}
    \end{sc}
}
\end{table}

\section{Conclusions}
In this paper, we proposed the Langevin Monte Carlo Thompson Sampling (\algname) algorithm, which uses Langevin Monte Carlo
to approximately sample the model parameter in contextual bandit problems from the posterior distribution. Unlike existing Thompson sampling based bandit algorithms that need to construct a Laplace approximation of the posterior which has a fixed approximation error, our method can directly sample from the posterior distribution with an arbitrarily small approximation error. Moreover, our algorithm only requires to perform noisy gradient descent updates which is more computationally efficient than sampling from a high dimensional non-isotropic Gaussian distribution. It is also worth noting that the proposed $\algname$ algorithm works for a large class of contextual bandit problems where the sampling distribution may not necessarily be a conjugate posterior and the reward models are general functions.

As a sanity check, we proved that $\algname$ enjoys the same regret upper bound as best Thompson sampling algorithms for linear contextual bandits, which theoretically demonstrates the competitiveness of $\algname$ in terms of regret minimization. Due to the efficiency and flexibility of our method, it would be an interesting and promising future direction to further investigate its theoretical performance in more general contextual bandits including generalized linear bandits and neural  bandits. 

We also conducted numerical experiments on both synthetic datasets and real-world datasets to empirically evaluate the performance of our method. We compared our method $\algname$ with various baseline algorithms in linear contextual bandits, generalized linear bandits, and neural contextual bandits. We observed that our method consistently outperforms existing baseline algorithms in these settings respectively. Moreover, our method is much more scalable to high dimensional problems where deep neural networks are used owing to its simplicity and efficiency. 

\section*{Acknowledgements}

The authors would like to thank the anonymous reviewers for their invaluable comments. PX is supported by PIMCO Postdoctoral Fellowship. AA is partially supported by Bren Named Chair Professorship at Caltech. 

\bibliography{example_paper}
\bibliographystyle{icml2022}

\newpage
\appendix
\onecolumn

\section{Proof of the Regret Bound for Linear Contextual Bandits}
In this section, we study the regret of Algorithm~\ref{alg:ts_lmc_general}. Before we analyze the regret of the proposed algorithm,
define $\xb_t^*=\argmax_{\xb\in\cX_t}\xb^{\top}\btheta^*$ to be the optimal arm at time step $t$. Define $\Delta_t(\xb)=\xb_{t}^*\btheta^*-\xb_{}^{\top}\btheta^*$ to be the gap between the expected reward of the best arm and an arbitrary arm $\xb\in\cX_t$ at time $t$. 

We follow a similar proof in~\citet{agrawal2013thompson} based on the notion of saturated and unsaturated arm sets, which are defined as follows respectively
\begin{align}
    \text{Saturated arms: }\cS_t&=\{\xb\in\cX_t:\Delta_t(\xb)> g_t(\xb)\},\label{eq:def_saturated}\\
    \text{Unsaturated arms: }\cU_t&=\{\xb\in\cX_t:\Delta_t(\xb)\leq g_t(\xb)\},\label{eq:def_unsaturated}
\end{align}
where $g_t(\xb)$ is a function of the time index $t$, arm $\xb$, and historical data up to time $t$, which will be determined in later proofs. 
Intuitively, saturated arms are already well estimated and the suboptimal gap of these arms are large enough such that they can be easily distinguished from the best arm. Define the filtration $\cF_t=\{\xb_1,r_1,\ldots,\xb_{t}\}$. The following proposition implies that the distribution of the iterate $\btheta_{t,K_t}$ of Algorithm~\ref{alg:ts_lmc_general} is a Gaussian distribution.
\begin{proposition}\label{lemma:distribution_of_theta_t_k}
Under filtration $\cF_{t-1}$, the parameter $\btheta_{t,K_t}$ used in the decision at round $t$ of Algorithm~\ref{alg:ts_lmc_general} follows a Gaussian distribution $\cN(\bmu_{t,K_t},\bSigma_{t,K_t})$, where the mean vector and the covariance matrix are defined as
\begin{align}
     \bmu_{t,K_t}&=\Ab_t^{K_t}\ldots\Ab_1^{k_1}\btheta_{1,0}+\sum_{i=1}^{t}\Ab_t^{K_t}\ldots\Ab_{i+1}^{k_{i+1}}\big(\Ib-\Ab_{i}^{K_t}\big)\hbtheta_{i},\\
     \bSigma_{t,K_t}&=\sum_{i=1}^{t}\frac{1}{\beta_i}\Ab_{t}^{K_t}\ldots\Ab_{i+1}^{k_{i+1}}\big(\Ib-\Ab_i^{2K_i}\big)\Vb_i(\Ib+\Ab_i)^{-1}\Ab_{i+1}^{k_{i+1}}\ldots\Ab_{t}^{K_t},
\end{align}
where $\Ab_i=\Ib-2\eta_{i}\Vb_{i}$ for $i=1,2,\ldots$.
\end{proposition}
It can be seen that the variance term depends more on the recent noise than those noises added in the history. 

The following lemma upper bounds the distance between the solution $\hbtheta_t$ of the ridge regression in \eqref{def:loss_time_t} and the true weight parameter $\btheta^*$ of the bandit model.
\begin{lemma}[\citet{agrawal2013thompson}]\label{lemma:concentration_events_regression}
Let $\delta\in(0,1)$. For any $\xb\in\RR^d$, it holds that
\begin{align}
    \PP\big(\big|\xb^{\top}(\hbtheta_t-\btheta^*)\big|&\leq \big(R\sqrt{d\log(t^3/\delta)}+1\big)\|\xb\|_{\Vb_t^{-1}}\big)\geq 1-\frac{\delta}{t^2}.
\end{align}
For the simplicity of notations, we denote  $g_R(t)=R\sqrt{d\log(t^3/\delta)}+1$ in the rest of the paper.
\end{lemma}
This can be bounded using the same bound of the least square solution in~\citet{agrawal2013thompson}. 

\begin{lemma}\label{lemma:subgaussian-tail}
Let $R>0$. For a $R$-subGaussian random variable $\xi$, it holds that $\PP(|\xi|>x)\leq \exp(1-x^2/R^2)$. In other words, $\PP(|\xi|>R\sqrt{1+2\log t})\leq1/t^2$ for any $t>1$. Define the following event 
\begin{align}\label{def:event_regression}
    E_{R,t}&=\big\{|\xb^{\top}\hbtheta_t-\xb^{\top}\btheta^*|\leq g_R(t)\|\xb\|_{\Vb_t^{-1}}, \forall\xb\in\cX_t\big\}\cap\{|\xi_t|\leq R\sqrt{1+2\log t}\}.
\end{align}
Then it holds that $\PP(E_{R,t})\geq 1-2/t^2$.
\end{lemma}

We define the event $E_{S,t}$ as follows, which provides an upper bound of the distance between the iterate $\btheta_{t,k}$ and the mode $\hbtheta_t$ of density $\pi_t\propto\exp(-\beta_t L_t(\btheta))$.

\begin{align}\label{def:event_sampling}
    E_{S,t}&=\Bigg\{|\xb^{\top}\btheta_{t,k}-\xb^{\top}\hbtheta_t|\leq \Bigg(5\sqrt{\frac{2d\log t}{3\beta_T}}+3/2\Bigg)\|\xb\|_{\Vb_t^{-1}}\Bigg\}.
\end{align}
The following lemma shows that $E_{S,t}$ happens with high probability.
\begin{lemma}\label{lemma:concentration_events_sampling}
Under event $E_{R,t}$, for any $\xb\in\cX_t$, with probability at least $1-1/t^2$ we have $\PP(E_{S,t})\geq 1-1/T^2$.
\end{lemma}

\begin{lemma}[Optimistic estimation]\label{lemma:optimism}
Based on the parameter estimation at time step $t$, the best arm is over estimated. In particular, let $K_j\geq\kappa_j\log(3R\sqrt{2dT\log(T^3/\delta)})$ and $1/\beta_j=4(R\sqrt{d\log(T^3/\delta)}+2)$ for all $j\in[T]$. Then conditional on event $E_{R,t}$, we have
\begin{align}
P\big(\xb_t^{*\top}\btheta_{t,k}>\xb_{t}^{*\top}\btheta^*\big)\geq \frac{1}{2\sqrt{2e\pi}}.
\end{align}
\end{lemma}

\begin{lemma}\label{lemma:prob_pull_unsaturated}
The probability of playing unsaturated arms is at least a constant. That is, 
\begin{align*}
    \PP(\xb_t\in\cU_t|E_{R,t})\geq \frac{1}{2\sqrt{2e\pi}}-\frac{1}{t^2}.
\end{align*}
\end{lemma}

The following lemma is a combination of Lemma 10 and Lemma 11 in~\citet{abbasi2011improved}, which is standard in the analysis of linear contextual bandits.
\begin{lemma}\label{lemma:det_sum}
Let $\{\xb_t\}_{t=1}^{\infty}$ be a sequence in $\RR^d$ and $\lambda>0$. Suppose $\|\xb_t\|_2\leq 1$ and $\lambda\geq 1$. Define $\Vb_t=\lambda\Ib+\sum_{s=1}^{t}\xb_t\xb_t^{\top}$. Then we have
\begin{align*}
    \det(\Vb_t)\leq(\lambda+t/d)^d,\quad\text{and }\sum_{t=1}^{T}\|\xb_t\|_{\Vb_{t-1}^{-1}}^2\leq2\log\frac{\det(\Vb_T)}{\det(\lambda\Ib)}\leq 2d\log(1+T/(\lambda d)).
\end{align*}
\end{lemma}

Recall the definition of in \eqref{eq:def_saturated} and \eqref{eq:def_unsaturated}, where $g_t(\xb)$ was not immediately specified. Now based on the knowledge of Lemmas~\ref{lemma:concentration_events_regression} and~\ref{lemma:concentration_events_sampling}, let us choose $g_t(\xb)$ as follows:
\begin{align}\label{def:choice_of_saturated_threshold}
    g_t(\xb)&=\big(R\sqrt{d\log(t^3/\delta)}+10Rd\sqrt{\log T\log(T^3/\delta)}+5/2\big)\|\xb_t\|_{\Vb_t^{-1}}.
\end{align}

\begin{proof}[Proof of Theorem~\ref{thm:regret_lcb}]
By definition, the regret of Algorithm~\ref{alg:ts_lmc_general} is $R(T)=\sum_{t=1}^{T}\Delta_{t}(\xb_t)$. 
At each time step, define $\bar\xb_t=\argmin_{\xb\in\cU_t}g_t(\xb)$ to be the arm in the unsaturated arm set that has the smallest value $\|\xb\|_{\Vb_t^{-1}}$. Then conditional on events $E_{R,t}$ and $E_{S,t}$, we have
\begin{align}\label{eq:bound_gap_with_unsaturatured_arms}
    \Delta_{t}(\xb_t)&=\xb_{t}^{*\top}\btheta^*-\xb_{t}^{\top}\btheta^*\notag\\
    &=\xb_{t}^{*\top}\btheta^*-\bar\xb_{t}^{\top}\btheta^*+\bar\xb_{t}^{\top}\btheta^*-\xb_{t}^{\top}\btheta^*\notag\\
    &\leq g_t(\bar\xb_t)+\big(\bar\xb_t^{\top}\btheta_{t,K_t}+g_t(\bar\xb_t)\big)-\big(\xb_t^{\top}\btheta_{t,K_t}-g_t(\xb_t)\big)\notag\\
    &\leq 2g_t(\bar\xb_t)+g_t(\xb_t),
\end{align}
where the first inequality is due to $\bar\xb_t\in\cU_t$, Lemma~\ref{lemma:concentration_events_regression} and Lemma~\ref{lemma:concentration_events_sampling} respectively, and the last inequality is due to the choice of $\xb_t$ in our algorithm.  We denote $p=1/(2\sqrt{2e\pi})$. Note that $g_t(\xb)>0$. We have
\begin{align}
    \EE[g_t(\xb_t)|\cF_t, E_{R,t}]&= \EE[g_t(\xb_t)|\cF_t, E_{R,t},  \xb_t\in\cU_t]\PP(\xb_t\in\cU_t)\notag\\
    &\qquad
    +\EE[g_t(\xb_t)|\cF_t, E_{R,t}, \xb_t\in\cS_t]\PP(\xb_t\in\cS_t)\notag\\
    &\geq (p-1/t^2)g_t(\bar\xb_t),
\end{align}
where the inequality holds due to the definition of $\bar\xb_t$ and Lemma~\ref{lemma:prob_pull_unsaturated}. Therefore, we have
\begin{align}
    \EE[\Delta_{t}(\xb_t)|\cF_t]&=\EE[\Delta_{t}(\xb_t)|\cF_t, E_{R,t}]\PP(E_{R,t})+\EE[\Delta_{t}(\xb_t)|\cF_t, E_R^c(t)]\PP(E_R^c(t))\notag\\
    &\leq \EE[2g_t(\bar\xb_t)+g_t(\xb_t)|\cF_t, E_{R,t}]\PP(E_{R,t})+\PP(E^c_R(t))\notag\\
    &\leq   \bigg(\frac{2}{p-1/t^2}+1\bigg)\EE[g_t(\xb_t)|\cF_t, E_{R,t}]+\frac{\delta}{t^2}\notag\\
    &\leq  \frac{c}{p}\EE[g_t(\xb_t)|\cF_t, E_{R,t}]+\frac{\delta}{t^2},
\end{align}
where we assume $\|\Delta_t(\xb)\|\leq 1$ for any $\xb\in\cX$ since both $\xb$ and $\btheta^*$ are bounded, and $c>0$ is a constant. Define $Y_t=\sum_{s=1}^{t}(\Delta_s(\xb_s)-c/p g_s(\xb_s)-\delta/s^2)$ and $Y_0=0$. Then we obtain $\EE[Y_{t}-Y_{t-1}|\cF_{t}]\leq0$, which implies that $\{Y_t\}_{t=0,1,\ldots}$ is a super martingale, corresponding a filtration $\cF_t$. Note that
\begin{align}
    |Y_t-Y_{t-1}|=|\Delta_t(\xb_t)-c/p g_t(\xb_t)-\delta/t^2|\leq  3c/p g_t(\xb_t).
\end{align}
Let $\epsilon^2=2\log (1/\delta)\sum_{t=1}^T9(c/p)^2 g_t(\xb_t)^2$. By Azuma-Hoeffding inequality in Lemma~\ref{lemma:azuma_martingale}, we know with probability at least $1-\delta$ it holds that
\begin{align*}
    Y_T=\sum_{t=1}^{T}(\Delta_t(\xb_t)-c/p g_t(\xb_t)-\delta/t^2)&\leq\sqrt{2\log (1/\delta)\sum_{t=1}^{T}9(c/p)^2 g_t(\xb_t)^2},
\end{align*}
which immediately implies with probability at least $1-\delta$ that
\begin{align*}
    R(T)\leq (1+3\sqrt{2\log(1/\delta)})c/p\sum_{t=1}^{T}g_t(\xb_t)+\sum_{t=1}^{T}\delta/t^2.
\end{align*}
Recall the definition of $g_t(\xb_t)$ in \eqref{def:choice_of_saturated_threshold}, we further have  with probability at least $1-\delta$ it holds
\begin{align*}
    R(T)&\leq (1+3\sqrt{2\log(1/\delta)})c/p\sum_{t=1}^{T}\big(R\sqrt{d\log(t^3/\delta)}+10Rd\sqrt{\log T\log(T^3/\delta)}+5/2\big)\|\xb_t\|_{\Vb_t^{-1}}+\frac{\pi^2\delta}{6}\notag\\
    &\leq C_0Rd\sqrt{dT\log T\log(T^3/\delta)\log(1/\delta)\log(1+T/(\lambda d))}\\
    &\leq C_0Rd\log(1/\delta)\sqrt{dT\log^3(1+T/(\lambda d))},
\end{align*}
where in the first inequality we used the fact that $\sum_{t=1}^{\infty}1/t^2=\pi^2/6$, the second inequality is due to Lemmas~\ref{lemma:optimism} and~\ref{lemma:det_sum} and Cauchy inequality, and $C_0$ is a constant independent of the problem.
\end{proof}

\section{Proof of Technical Lemmas}

\subsection{Proof of Proposition~\ref{lemma:distribution_of_theta_t_k}}
\begin{proof}[Proof of Proposition~\ref{lemma:distribution_of_theta_t_k}]
By the updating rule of $\btheta_{t,K_t}$ in Algorithm~\ref{alg:ts_lmc_general}, we have
\begin{align}\label{eq:expand_theta_tk_one_epoch}
    \btheta_{t,K_t}&=\btheta_{t,K_t-1}-2\eta_{t}(\Vb_t\btheta_{t,K_t-1}-\bbb_t)+\sqrt{2\eta_{t}\beta_t^{-1}}\epsilon_{t,K_t}\notag\\
    &=(\Ib-2\eta_{t}\Vb_t)\btheta_{t,K_t-1}+2\eta_{t}\bbb_t+\sqrt{2\eta_{t}\beta_t^{-1}}\epsilon_{t,K_t}\notag\\
    &=(\Ib-2\eta_{t}\Vb_t)^{K_t}\btheta_{t,0}+\sum_{l=0}^{K_t-1}(\Ib-2\eta_{t}\Vb_t)^l \big(2\eta_{t}\bbb_t+\sqrt{2\eta_{t}\beta_t^{-1}}\epsilon_{t,k-l}\big)\notag\\
    &=(\Ib-2\eta_{t}\Vb_t)^{K_t}\btheta_{t,0}+2\eta_{t}\sum_{l=0}^{K_t-1}(\Ib-2\eta_t\Vb_t)^l \bbb_t+\sqrt{2\eta_{t}\beta_t^{-1}}\sum_{l=0}^{K_t-1}(\Ib-2\eta_{t}\Vb_t)^l\epsilon_{t,K_t-l}.
\end{align}
Recall that in the inner loop of Algorithm~\ref{alg:ts_lmc_general}, we use a warm start from previous iteration, namely $\btheta_{t,0}=\btheta_{t-1,k_{t-1}}$. To simplify the notation, we denote $\Ab_i=\Ib-2\eta_{i}\Vb_{i}$ for $i=1,2,\ldots$. Note that $\Ab_i$ is symmetric and satisfies $\Ib\succ\Ib-2\eta_i\Vb_i\succ\zero$ if the step size is chosen such that $0<\eta_i<1/(2\lambda_{\max}(\Vb_i))$. Therefore, we further have
\begin{align*}
    \btheta_{t,K_t}
    &=\Ab_t^{K_t}\btheta_{t-1,k_{t-1}}+\big(\Ib-\Ab_t^{K_t}\big)\hbtheta_t+\sqrt{2\eta_{t}\beta_t^{-1}}\sum_{l=0}^{K_t-1}\Ab_t^{l}\epsilon_{t,K_t-l}\\
    &=\Ab_t^{K_t}\ldots\Ab_1^{k_1}\btheta_{1,0}+\sum_{i=1}^{t}\Ab_t^{K_t}\ldots\Ab_{i+1}^{k_{i+1}}\big(\Ib-\Ab_i^{K_t}\big)\hbtheta_{i}+\sum_{i=1}^{t}\sqrt{\frac{2\eta_i}{\beta_i}}\Ab_t^{K_t}\ldots\Ab_{i+1}^{k_{i+1}}\bigg(\sum_{l=0}^{K_t-1}\Ab_i^{l}\bepsilon_{i,K_i-l}\bigg),
\end{align*}
where in the first equality we used
$\Ib+\Ab+\ldots+\Ab^{n-1}=(\Ib-\Ab^n)(\Ib-\Ab)^{-1}$ and $\Vb_t^{-1}\bbb_t=\hbtheta_t$. Conditional on $\cF_{t-1}$ and the initialization $\btheta_{1,0}$, we know that $\btheta_{t,K_t}$ follows the Gaussian distribution $\cN(\bmu_{t,K_t},\bSigma_{t,K_t})$. Based on the property of multivariate Gaussian distribution, if $\bepsilon\sim\cN(\zero,\Ib_{d\times d})$, then we have $\Ab\bepsilon+\bmu\sim\cN(\bmu,\Ab\Ab^{\top})$ for any $\Ab\in\RR^{d\times d}$ and $\bmu\in\RR^d$. Then the mean vector is defined as
\begin{align}
    \bmu_{t,K_t}&=\Ab_t^{K_t}\ldots\Ab_1^{k_1}\btheta_{1,0}+\sum_{i=1}^{t}\Ab_t^{K_t}\ldots\Ab_{i+1}^{k_{i+1}}\big(\Ib-\Ab_{i}^{K_t}\big)\hbtheta_{i}.\label{eq:mu_theta_tk}
\end{align}
Similarly, the covariance matrix is defined as \begin{align}
    \bSigma_{t,K_t}&=\sum_{i=1}^{t}\frac{2\eta_i}{\beta_i}\Ab_{t}^{K_t}\ldots\Ab_{i+1}^{k_{i+1}}\sum_{l=0}^{K_i-1}\Ab_i^{2l}\Ab_{i+1}^{k_{i+1}}\ldots\Ab_{t}^{K_t}\notag\\
    &=\sum_{i=1}^{t}\frac{2\eta_i}{\beta_i}\Ab_{t}^{K_t}\ldots\Ab_{i+1}^{k_{i+1}}\big(\Ib-\Ab_i^{2K_i}\big)(\Ib-\Ab_i^2)^{-1}\Ab_{i+1}^{k_{i+1}}\ldots\Ab_{t}^{K_t}\label{eq:sigma_theta_tk}  \\
    &=\sum_{i=1}^{t}\frac{1}{\beta_i}\Ab_{t}^{K_t}\ldots\Ab_{i+1}^{k_{i+1}}\big(\Ib-\Ab_i^{2K_i}\big)\Vb_i(\Ib+\Ab_i)^{-1}\Ab_{i+1}^{k_{i+1}}\ldots\Ab_{t}^{K_t},
\end{align}
where we used the fact that $\Ib+\Ab+\Ab^2+\ldots+\Ab^{n-1}=(\Ib-\Ab^n)(\Ib-\Ab)^{-1}=(\Ib-\Ab)^{-1}(\Ib-\Ab^n)$ if $\Ab$ is symmetric.
\end{proof}

\subsection{Proof of Lemma~\ref{lemma:subgaussian-tail}}
\begin{proof}[Proof of Lemma~\ref{lemma:subgaussian-tail}]
The first statement can be proved by standard concentration techniques ~\citep{vershynin2010introduction}. Note that
\begin{align*}
    \PP(E_{R,t}^c)&=\PP\big(\{|\xb^{\top}\hbtheta_t-\xb^{\top}\btheta^*|\leq g_R(t)\|\xb\|_{\Vb_t^{-1}}, \forall\xb\in\cX_t\}\cup\{|\xi_t|\leq R\sqrt{1+2\log t}\}\big)\\
    &\leq\PP\big(\{|\xb^{\top}\hbtheta_t-\xb^{\top}\btheta^*|\leq g_R(t)\|\xb\|_{\Vb_t^{-1}}, \forall\xb\in\cX_t\}\big)+\PP\big(\{|\xi_t|\leq R\sqrt{1+2\log t}\}\big).
\end{align*}
 Substituting the first statement and the result in Lemma~\ref{lemma:concentration_events_regression} into the above inequality, we know that $\PP(E_{R,t})\geq 1-2/t^2$.
\end{proof}

\subsection{Proof of Lemma~\ref{lemma:concentration_events_sampling}}
\begin{proof}[Proof of Lemma~\ref{lemma:concentration_events_sampling}]
For any $\xb\in\RR^d$, we can decompose $\xb^{\top}(\btheta_{t,k}-\hbtheta_t)$ as follows.
\begin{align}\label{eq:thetak_htheta_t_decomp}
    \xb^{\top}(\btheta_{t,k}-\hbtheta_t)=\xb^{\top}(\btheta_{t,k}-\bmu_{t,k})+\xb^{\top}(\bmu_{t,k}-\hbtheta_t).
\end{align}
\noindent\textbf{Bounding term $\xb^{\top}(\btheta_{t,k}-\bmu_{t,k})$:} we have
\begin{align*}
    \big|\xb_{}^{\top}\btheta_{t,k}-\xb_{}^{\top}\bmu_{t,k}\big|\leq\big\|\xb^{\top}\bSigma_{t,K_t}^{1/2}\big\|_2\big\|\bSigma_{t,K_t}^{-1/2}(\btheta_{t,k}-\bmu_{t,k})\big\|_2.
\end{align*}
According to Proposition~\ref{lemma:distribution_of_theta_t_k},  $\bSigma_{t,K_t}^{-1/2}(\btheta_{t,k}-\bmu_{t,k})$ is a standard Gaussian random vector in $\RR^d$. Therefore, we have
\begin{align}\label{eq:gaussian_vector_tail}
    \PP\Big(\big\|\bSigma_{t,K_t}^{-1/2}(\btheta_{t,k}-\bmu_{t,k})\big\|_2\geq \sqrt{4d\log t}\Big)\geq \frac{1}{t^2}.
\end{align}
Note that when we choose $\eta_i\leq 1/(4\lambda_{\max}(\Vb_i))$ for all $i$,  we have
\begin{align}\label{eq:def_contraction_matrix_step_size}
\begin{split}
    \frac{1}{2}\Ib&\prec\Ab_i=\Ib-2\eta_i\Vb_i\prec (1-2\eta_i\lambda_{\min}(\Vb_i))\Ib,\\
    \frac{3}{2}\Ib&\prec \Ib+\Ab_i=2\Ib-2\eta_i\Vb_i \prec2\Ib.
\end{split}
\end{align}
By definition $\Ab_i=\Ib-2\eta_i\Vb_i$ and $\Vb_i$ is symmetric. Therefore, $\Ab_i$ and $\Vb_i^{-1}$ commute, which implies
\begin{align}
    \Ab_i^{2K_i}\Vb_i^{-1}&=(\Ib-2\eta_i\Vb_i)\ldots(\Ib-2\eta_i\Vb_i)(\Ib-2\eta_i\Vb_i)\Vb_i^{-1}\notag\\
    &=(\Ib-2\eta_i\Vb_i)\ldots(\Ib-2\eta_i\Vb_i)\Vb_i^{-1}(\Ib-2\eta_i\Vb_i)\notag\\
    &=\Ab_i^{K_i}\Vb_i^{-1}\Ab_{i}^{K_i}.
\end{align}
Recall the definition of $\bSigma_{t,K_t}$ in Proposition~\ref{lemma:distribution_of_theta_t_k},  we have
\begin{align}
     \xb^{\top}\bSigma_{t,K_t}\xb&=\sum_{i=1}^{t}\frac{1}{\beta_i}\xb^{\top}\Ab_{t}^{K_t}\ldots\Ab_{i+1}^{k_{i+1}}\big(\Ib-\Ab_i^{2K_i}\big)\Vb_i^{-1}(\Ib+\Ab_i)^{-1}\Ab_{i+1}^{k_{i+1}}\ldots\Ab_{t}^{K_t}\xb\notag\\
     &\leq \sum_{i=1}^{t}\frac{2}{3\beta_i}\xb^{\top}\Ab_{t}^{K_t}\ldots\Ab_{i+1}^{k_{i+1}}\big(\Vb_i^{-1}-\Ab_i^{K_i}\Vb_i^{-1}\Ab_i^{K_i}\big)\Ab_{i+1}^{k_{i+1}}\ldots\Ab_{t}^{K_t}\xb^{}\notag\\
     &=\frac{2}{3\beta_T}\sum_{i=1}^{t-1}\xb^{\top}\Ab_{t}^{K_t}\ldots\Ab_{i+1}^{k_{i+1}}\big(\Vb_{i}^{-1}-\Vb_{i+1}^{-1}\big)\Ab_{i+1}^{k_{i+1}}\ldots\Ab_{t}^{K_t}\xb^{}\notag\\
     &\qquad-\frac{2}{3\beta_T}\xb^{\top}\Ab_{t}^{K_t}\ldots\Ab_{1}^{k_{1}}\Vb_1^{-1}\Ab_{1}^{k_1}\ldots\Ab_{t}^{K_t}\xb^{} +\frac{2}{3\beta_T}\xb^{\top}\Vb_t^{-1}\xb,
\end{align}
where the fist inequality is due to \eqref{eq:def_contraction_matrix_step_size}, and the last equality is due to the choice of $1/\beta_i=1/\beta_T$ for all $i$. 
By the definition in \eqref{eq:V_b_update} and Sherman-Morrison formula, we have
\begin{align}
    \Vb_{i}^{-1}-\Vb_{i+1}^{-1}=\Vb_{i}^{-1}-\big(\Vb_i+\xb_{i}\xb_i^{\top}\big)^{-1}=\frac{\Vb_i^{-1}\xb_{i}\xb_i^{\top}\Vb_i^{-1}}{1+\|\xb_i\|_{\Vb_i^{-1}}^2},
\end{align}
which immediately implies
\begin{align}
    \xb^{\top}\Ab_{t}^{K_t}\ldots\Ab_{i+1}^{k_{i+1}}\big(\Vb_{i}^{-1}-\Vb_{i+1}^{-1}\big)\Ab_{i+1}^{k_{i+1}}\ldots\Ab_{t}^{K_t}\xb^{}
    &=\xb^{\top}\Ab_{t}^{K_t}\ldots\Ab_{i+1}^{k_{i+1}}\frac{\Vb_i^{-1}\xb_{i}\xb_i^{\top}\Vb_i^{-1}}{1+\|\xb_i\|_{\Vb_i^{-1}}^2}\Ab_{i+1}^{k_{i+1}}\ldots\Ab_{t}^{K_t}\xb^{}\notag\\
    &\leq \big(\xb^{\top}\Ab_{t}^{K_t}\ldots\Ab_{i+1}^{k_{i+1}}\Vb_i^{-1}\xb_{i}\big)^2\notag\\
    &\leq\|\Ab_{t}^{K_t}\ldots\Ab_{i+1}^{k_{i+1}}\Vb_i^{-1/2}\xb\|_2^2\cdot\|\Vb_i^{-1/2}\xb_i\|_2^2\notag\\
    &\leq \prod_{j=i+1}^{t}(1-2\eta_j\lambda_{\min}(\Vb_j))^{2K_j}\|\xb_i\|_{\Vb_i^{-1}}^2\|\xb\|_{\Vb_i^{-1}}^2,
\end{align}
where we used $0<1/i\leq\|\xb\|_{\Vb_i^{-1}}\leq1$ and the last inequality is due to \eqref{eq:def_contraction_matrix_step_size}. Therefore, we have
\begin{align}
    \xb^{\top}\bSigma_{t,k}\xb &\leq  \frac{2}{3\beta_T}\|\xb\|_{\Vb_t^{-1}}^2
    +\frac{2}{3\beta_T}\sum_{i=1}^{t-1}\prod_{j=i+1}^{t}(1-2\eta_j\lambda_{\min}(\Vb_j))^{2K_j}\|\xb_i\|_{\Vb_i^{-1}}^2\|\xb\|_{\Vb_i^{-1}}^2,
\end{align}
which immediately implies
\begin{align*}
    \|\xb\|_{\bSigma_{t,K_t}}\leq \sqrt{\frac{2}{3\beta_T}}\bigg(\|\xb\|_{\Vb_t^{-1}}+\sum_{i=1}^{t-1}\prod_{j=i+1}^{t}(1-2\eta_j\lambda_{\min}(\Vb_j))^{K_j}\|\xb_i\|_{\Vb_i^{-1}}\|\xb\|_{\Vb_i^{-1}}\bigg):=\hat g_t(\xb).
\end{align*}
Therefore, it holds that
\begin{align}\label{eq:thetak_htheta_t_decomp_term1}
    &\PP\big(\big|\xb_{}^{\top}\btheta_{t,k}-\xb_{}^{\top}\bmu_{t,K_t}\big|\geq  2\hat g_t(\xb)\sqrt{d\log t}\big)\notag\\
    &\leq\PP\big(\big|\xb_{}^{\top}\btheta_{t,k}-\xb_{}^{\top}\bmu_{t,K_t}\big|\geq 2\sqrt{d\log t}\|\xb\|_{\bSigma_{t,K_t}}\big)\notag\\
    &\leq\PP\big(\big\|\xb^{\top}\bSigma_{t,K_t}^{1/2}\big\|_2\big\|\bSigma_{t,K_t}^{-1/2}(\btheta_{t,k}-\bmu_{t,k})\big\|_2\geq 2\sqrt{d\log t}\|\xb\|_{\bSigma_{t,K_t}}\big) \notag\\
    &\leq\frac{1}{t^2},
\end{align}
where the last inequality follows from \eqref{eq:gaussian_vector_tail}.

\noindent\textbf{Bounding term $\xb^{\top}(\bmu_{t,k}-\hbtheta_t)$:}
recall  $\bmu_{t,k}$ defined in  \eqref{eq:mu_theta_tk}, we have
\begin{align}
    \bmu_{t,K_t}&=\Ab_t^{K_t}\ldots\Ab_1^{k_1}\btheta_{1,0}+\sum_{i=1}^{t}\Ab_t^{K_t}\ldots\Ab_{i+1}^{k_{i+1}}\big(\Ib-\Ab_{i}^{K_t}\big)\hbtheta_{i}\notag\\
    &=\Ab_t^{K_t}\ldots\Ab_1^{k_1}\btheta_{1,0}+\sum_{i=1}^{t-1}\Ab_t^{K_t}\ldots\Ab_{i+1}^{k_{i+1}}\big(\hbtheta_{i}-\hbtheta_{i+1}\big)-\Ab_t^{K_t}\ldots\Ab_1^{k_1}\hbtheta_1+\hbtheta_t\notag\\
    &=\Ab_t^{K_t}\ldots\Ab_1^{k_1}(\btheta_{1,0}-\hbtheta_1)+\sum_{i=1}^{t-1}\Ab_t^{K_t}\ldots\Ab_{i+1}^{k_{i+1}}\big(\hbtheta_{i}-\hbtheta_{i+1})+\hbtheta_t,
\end{align}
which immediately implies
\begin{align}
    \xb^{\top}(\bmu_{t,K_t}-\hbtheta_t)&=\underbrace{\xb^{\top}\Ab_t^{K_t}\ldots\Ab_1^{k_1}\big(\btheta_{1,0}-\hbtheta_1\big)}_{I_1}+\underbrace{\xb^{\top}\sum_{i=1}^{t-1}\Ab_t^{K_t}\ldots\Ab_{i+1}^{k_{i+1}}\big(\hbtheta_{i}-\hbtheta_{i+1})}_{I_2}.
\end{align}
For term $I_1$, recall the inequalities in \eqref{eq:def_contraction_matrix_step_size}, we have
\begin{align}
   \xb^{\top} \Ab_t^{K_t}\ldots\Ab_1^{k_1}\big(\btheta_{1,0}-\hbtheta_1\big)
    &\leq\prod_{i=1}^{t}(1-2\eta_i\lambda_{\min}(\Vb_i))^{K_i}\|\xb\|_2\|\btheta_{1,0}-\hbtheta_1\|_2.
\end{align}
Recall that in Algorithm~\ref{alg:ts_lmc_general}, we choose $\btheta_{1,0}=\zero$ and $\hbtheta_1=\Vb_1^{-1}\bbb_1=\zero$. We have $I_1=0$.
Now let's look at term $I_2$. Note that by Sherman–Morrison formula we have
\begin{align*}
    \hbtheta_{t+1}-\hbtheta_{t}&=\Vb_{t+1}^{-1}\bbb_{t+1}-\Vb_{t}^{-1}\bbb_{t}\\
    &=\frac{\Vb_{t}^{-1}\xb_{t}\big(r_t-\xb_{t}^{\top}\hbtheta_{t}\big)}{1+\xb_{t}^{\top}\Vb_{t}^{-1}\xb_{t}}\\
    &=\frac{\Vb_{t}^{-1}\xb_{t}\big(\xb_t^{\top}\btheta^*-\xb_{t}^{\top}\hbtheta_{t}\big)}{1+\xb_{t}^{\top}\Vb_{t}^{-1}\xb_{t}}+\frac{\Vb_{t}^{-1}\xb_{t}\xi_t}{1+\xb_{t}^{\top}\Vb_{t}^{-1}\xb_{t}}.
\end{align*}
Therefore, $I_2$ can be bounded as follows.
\begin{align}
   I_2&\leq \Bigg|\xb^{\top}\sum_{i=1}^{t-1}\Ab_t^{K_t}\ldots\Ab_{i+1}^{k_{i+1}}\Vb_i^{-1}\xb_i\frac{\xi_i+\xb_i^{\top}\big(\btheta^*-\hbtheta_t\big)}{1+\xb_i^{\top}\Vb_i^{-1}\xb_i}\Bigg|\notag\\
   &\leq \sum_{i=1}^{t-1}\big|\xb^{\top}\Ab_t^{K_t}\ldots\Ab_{i+1}^{k_{i+1}}\Vb_i^{-1}\xb_i\big|\Bigg|\frac{\xi_i+\xb_i^{\top}\big(\btheta^*-\hbtheta_t\big)}{1+\xb_i^{\top}\Vb_i^{-1}\xb_i}\Bigg|\notag\\
   &\leq \sum_{i=1}^{t-1}\big|\xb^{\top}\Ab_t^{K_t}\ldots\Ab_{i+1}^{k_{i+1}}\Vb_i^{-1}\xb_i\big| \big[|\xi_i|+\big|\xb_i^{\top}\big(\btheta^*-\hbtheta_t\big)\big|\big]\notag\\
   &\leq \sum_{i=1}^{t-1}\|\Ab_t^{K_t}\ldots\Ab_{i+1}^{k_{i+1}}\xb\|_{\Vb_i^{-1}}\|\xb_i\|_{\Vb_i^{-1}}\big| \big[|\xi_i|+\big|\xb_i^{\top}\big(\btheta^*-\hbtheta_t\big)\big|\big]\notag\\
   &\leq \sum_{i=1}^{t-1}\prod_{j=i+1}^{t}(1-2\eta_{j}\lambda_{\min}(\Vb_{j}))^{K_j}\|\xb\|_{\Vb_i^{-1}}\|\xb_i\|_{\Vb_i^{-1}}\big| \big(|\xi_i|+\big|\xb_i^{\top}\big(\btheta^*-\hbtheta_t\big)\big|\big).
\end{align}
Note that under event $E_{R,t}$, by Lemma~\ref{lemma:subgaussian-tail}, we have $\big|\xb_i^{\top}\big(\btheta^*-\hbtheta_t\big)\big|\leq g_R(t)\|\xb_i\|_{\Vb_t^{-1}}\leq g_R(t)$ and that $|\xi_i|\leq R\sqrt{1+2\log t}$. Combining the above results, we have
\begin{align}\label{eq:thetak_htheta_t_decomp_term2}
    |\xb^{\top}(\bmu_{t,K_t}-\hbtheta_t)|&\leq 
    \sum_{i=1}^{t-1}\prod_{j=i+1}^{t}(1-2\eta_{j}\lambda_{\min}(\Vb_{j}))^{K_j}\|\xb_i\|_{\Vb_i^{-1}}\|\xb\|_{\Vb_i^{-1}} \big(R\sqrt{1+2\log t}+g_{R}(i)\big).
\end{align}
Substituting \eqref{eq:thetak_htheta_t_decomp_term1} and \eqref{eq:thetak_htheta_t_decomp_term2} into \eqref{eq:thetak_htheta_t_decomp}, we have with probability at least $1-1/t^2$ that
\begin{align}\label{eq:thetak_htheta_errorbound_original}
    |\xb^{\top}(\btheta_{t,K_t}-\hbtheta_t)|&\leq  
    \sum_{i=1}^{t-1}\prod_{j=i+1}^{t}(1-2\eta_{j}\lambda_{\min}(\Vb_{j}))^{K_j} \big(R\sqrt{1+2\log t}+g_{R}(i)\big)\|\xb_i\|_{\Vb_i^{-1}}\|\xb\|_{\Vb_i^{-1}}\notag\\
    &\qquad+2\sqrt{\frac{2d\log t}{3\beta_T}}\bigg(\|\xb\|_{\Vb_t^{-1}}+\sum_{i=1}^{t-1}\prod_{j=i+1}^{t}(1-2\eta_j\lambda_{\min}(\Vb_j))^{K_j}\|\xb_i\|_{\Vb_i^{-1}}\|\xb\|_{\Vb_i^{-1}}\bigg)\notag\\
    &:=Q,
\end{align}
where we denote the upper bound as $Q$. For any $j\in[t]$, recall that we require  $\eta_j<1/(2\lambda_{\max}(\Vb_j))$. Now we choose $\eta_j=1/(4\lambda_{\max}(\Vb_j))$. Then it holds that
\begin{align*}
    (1-2\eta_j\lambda_{\min}(\Vb_j))^{K_j}=(1-1/(2\kappa_j))^{K_j},
\end{align*}
where $\kappa_j=\lambda_{\max}(\Vb_j)/\lambda_{\min}(\Vb_j)$. In order to ensure the above quantity be smaller than $\epsilon$, we need
\begin{align}
    K_j\geq \frac{\log (1/\epsilon)}{\log\frac{1}{1-1/(2\kappa_j)}}.
\end{align}
Note that $e^{-x}>1-x$ for any $0<x<1$. Since $1/(2\kappa_j)\leq 1/2$, we have $\log(1/(1-1/(2\kappa_j)))\geq 1/(2\kappa_j)$. Therefore, it suffices to set $K_j\geq 2\kappa_j\log(1/\epsilon)$ to ensure $(1-1/(2\kappa_j))^{K_j}\leq \epsilon$. Recall the definition $g_R(t)=R\sqrt{d\log(t^3/\delta)}+1$. Then we have 
\begin{align*}
    R\sqrt{1+2\log t}+g_{R}(i)&\leq R\sqrt{2\log t}+R\sqrt{d\log(t^3/\delta)}+R+1\\
    &\leq 3R\sqrt{2d\log(t^3/\delta)}.
\end{align*}
By setting $\epsilon=(3R\sqrt{2dt\log(t^3/\delta)})^{-1}$, we obtain 
\begin{align}
    Q&\leq\sum_{i=1}^{t-1}\epsilon^{t-i}3R\sqrt{2d\log(t^3/\delta)}\|\xb\|_{2}+2\sqrt{\frac{2d\log t}{3\beta_T}}\bigg(\|\xb\|_{\Vb_t^{-1}}+\sum_{i=1}^{t-1}\epsilon^{t-i}\|\xb\|_{2}\bigg)\notag\\
    &\leq \sum_{i=1}^{t-1}\epsilon^{t-1-i}\|\xb\|_{\Vb_t^{-1}}+ 2\sqrt{\frac{2d\log t}{3\beta_T}}\bigg(\|\xb\|_{\Vb_t^{-1}}+\sum_{i=1}^{t-1}\epsilon^{t-1-i}\|\xb\|_{\Vb_t^{-1}}\bigg)\notag \\
    &\leq\bigg(5\sqrt{\frac{2d\log t}{3\beta_T}}+3/2\bigg)\|\xb\|_{\Vb_t^{-1}},
\end{align}
where the second inequality is due to $\|\xb\|_{\Vb_t^{-1}}\geq 1/\sqrt{t}\|\xb\|_2$, and the third inequality is due to $\sum_{i=1}^{t-1}\epsilon^{t-1-i}=\sum_{i=0}^{t-2}\epsilon^i<1/(1-\epsilon)\leq 3/2$. Therefore, it holds that 
\begin{align*}
    \PP(E_{S,t})
    &\geq
    \PP\big(\big|\xb^{\top}(\btheta_{t,k}-\hbtheta_t)\big|\leq Q\big)\\
    &\geq 1-1/t^2,
\end{align*}
where the last inequality is due to \eqref{eq:thetak_htheta_errorbound_original}.
\end{proof}

\subsection{Proof of Lemma~\ref{lemma:optimism}}
\begin{proof}[Proof of Lemma~\ref{lemma:optimism}]
Based on the mean and covariance matrix defined in \eqref{eq:mu_theta_tk} and \eqref{eq:sigma_theta_tk}, we have that $\xb_{t}^{*\top}\btheta_{t,k}$ follows the distribution $\cN(\xb_{t}^{*\top}\bmu_{t,k},\xb_{t}^{*\top}\bSigma_{t,k}\xb_{t}^*)$. By Lemma~\ref{lem:gaussian-tail}, we have
\begin{align}
P\big(\xb_{t}^{*\top}\btheta_{t,k}>\xb_{t}^{*\top}\btheta^*\big)&=P\Bigg(\frac{\xb_{t}^{*\top}\btheta_{t,k}-\xb_{t}^{*\top}\bmu_{t,k}}{\sqrt{\xb_{t}^{*\top}\bSigma_{t,k}\xb_{t}^{*}}}>\frac{\xb_{t}^{*\top}\btheta^*-\xb_{t}^{*\top}\bmu_{t,k}}{\sqrt{\xb_{t}^{*\top}\bSigma_{t,k}\xb_{t}^{*}}}\Bigg)\notag\\
&\geq\frac{1}{2\sqrt{2\pi}}e^{-Z_t^2/2},
\end{align}
where the inequality holds when $|Z_t|<1$ and we define $Z_t=\big(\xb_{t}^{*\top}\btheta^*-\xb_{t}^{*\top}\bmu_{t,k}\big)/\sqrt{\xb_t^{*\top}\bSigma_{t,k}\xb_t^{*}}$. In the rest of the proof, we will show that $|Z_t|<1$ under event $E_{R,t}$. First note that by triangle inequality we have
\begin{align}
|\xb_{t}^{*\top}\btheta^*-\xb_{t}^{*\top}\bmu_{t,k}|&\leq |\xb_{t}^{*\top}\btheta^*-\xb_{t}^{*\top}\hbtheta_t|+|\xb_{t}^{*\top}\hbtheta_t-\xb_{t}^{*\top}\bmu_{t,k}|\notag\\
&\leq \sum_{i=1}^{t-1}\prod_{j=i+1}^{t}(1-2\eta_{j}\lambda_{\min}(\Vb_{j}))^{K_j}\|\xb_i\|_{\Vb_i^{-1}}\|\xb_t^{*}\|_{\Vb_i^{-1}} \big(R\sqrt{1+2\log t}+g_{R}(i)\big)\notag\\
&\qquad+g_R(t)\|\xb_{t}^{*}\|_{\Vb_t^{-1}},
\end{align}
where in the second inequality we used the conclusion in Lemma~\ref{lemma:subgaussian-tail} since event $E_{R,t}$ holds and \eqref{eq:thetak_htheta_t_decomp_term2}. When we choose $K_j\geq \kappa_j\log(3R\sqrt{2dt\log(t^3/\delta)})$, we further have
\begin{align}\label{eq:optimisim_numerator}
    |\xb_{t}^{*\top}\btheta^*-\xb_{t}^{*\top}\bmu_{t,k}|&\leq\sum_{i=1}^{t-1} t^{-(t-i)}\|\xb_t^*\|_2\big(R\sqrt{1+2\log t}+g_{R}(t)\big)+g_R(t)\|\xb_{t}^{*}\|_{\Vb_t^{-1}}\notag\\
    &\leq \frac{1}{3R\sqrt{2dt\log(t^3/\delta)}}\|\xb_t^*\|_2\big(R\sqrt{1+2\log t}+g_{R}(t)\big)+g_R(t)\|\xb_{t}^{*}\|_{\Vb_t^{-1}}\notag\\
    &\leq (R\sqrt{d\log(t^3/\delta)}+2)\|\xb_{t}^{*}\|_{\Vb_t^{-1}},
\end{align}
where in the last inequality we used the fact that $g_R(t)=R\sqrt{d\log(t^3/\delta)}+1$.

On the other hand, recall the definition of $\bSigma_{t,K_t}$ in Proposition~\ref{lemma:distribution_of_theta_t_k}. Following similar proof as in the previous lemma, we have 
\begin{align}
    \xb_t^{*\top}\bSigma_{t,k}\xb_t^{*}&=\sum_{i=1}^{t}\frac{1}{\beta_i}\xb_t^{*\top}\Ab_{t}^{K_t}\ldots\Ab_{i+1}^{k_{i+1}}\big(\Ib-\Ab_i^{2K_i}\big)\Vb_i^{-1}(\Ib+\Ab_i)^{-1}\Ab_{i+1}^{k_{i+1}}\ldots\Ab_{t}^{K_t}\xb_t^{*}\notag\\
    &\geq \sum_{i=1}^{t}\frac{1}{2\beta_i}\xb_t^{*\top}\Ab_{t}^{K_t}\ldots\Ab_{i+1}^{k_{i+1}}\big(\Ib-\Ab_i^{2K_i}\big)\Vb_i^{-1}\Ab_{i+1}^{k_{i+1}}\ldots\Ab_{t}^{K_t}\xb_t^{*}.
\end{align}
Note that by definition $\Ab_i=\Ib-2\eta_i\Vb_i$ and $\Vb_i$ is symmetric. Therefore, $\Ab_i$ and $\Vb_i^{-1}$ commute, and it holds that
\begin{align}
    \Ab_i^{2K_i}\Vb_i^{-1}&=(\Ib-2\eta_i\Vb_i)\ldots(\Ib-2\eta_i\Vb_i)(\Ib-2\eta_i\Vb_i)\Vb_i^{-1}\notag\\
    &=(\Ib-2\eta_i\Vb_i)\ldots(\Ib-2\eta_i\Vb_i)\Vb_i^{-1}(\Ib-2\eta_i\Vb_i)\notag\\
    &=\Ab_i^{K_i}\Vb_i^{-1}\Ab_{i}^{K_i}.
\end{align}
Hence we have
\begin{align}
     \xb_t^{*\top}\bSigma_{t,k}\xb_t^{*}&\geq \sum_{i=1}^{t}\frac{1}{2\beta_i}\xb_t^{*\top}\Ab_{t}^{K_t}\ldots\Ab_{i+1}^{k_{i+1}}\big(\Vb_i^{-1}-\Ab_i^{K_i}\Vb_i^{-1}\Ab_i^{K_i}\big)\Ab_{i+1}^{k_{i+1}}\ldots\Ab_{t}^{K_t}\xb_t^{*}\notag\\
     &=\frac{1}{2\beta_T}\sum_{i=1}^{t-1}\xb_t^{*\top}\Ab_{t}^{K_t}\ldots\Ab_{i+1}^{k_{i+1}}\big(\Vb_{i}^{-1}-\Vb_{i+1}^{-1}\big)\Ab_{i+1}^{k_{i+1}}\ldots\Ab_{t}^{K_t}\xb_t^{*}\notag\\
     &\qquad-\frac{1}{2\beta_T}\xb_t^{*\top}\Ab_{t}^{K_t}\ldots\Ab_{1}^{k_{1}}\Vb_1^{-1}\Ab_{1}^{k_1}\ldots\Ab_{t}^{K_t}\xb_t^{*} +\frac{1}{2\beta_T}\xb_t^{*\top}\Vb_t^{-1}\xb_t^*.
\end{align}
where we used the choice of $1/\beta_i=1/\beta_T$ for all $i$. 
By the definition in \eqref{eq:V_b_update} and Sherman-Morrison formula, we have
\begin{align}
    \Vb_{i}^{-1}-\Vb_{i+1}^{-1}=\Vb_{i}^{-1}-\big(\Vb_i+\xb_{i}\xb_i^{\top}\big)^{-1}=\frac{\Vb_i^{-1}\xb_{i}\xb_i^{\top}\Vb_i^{-1}}{1+\|\xb_i\|_{\Vb_i^{-1}}^2},
\end{align}
which immediately implies
\begin{align}
    \xb_t^{*\top}\Ab_{t}^{K_t}\ldots\Ab_{i+1}^{k_{i+1}}\big(\Vb_{i}^{-1}-\Vb_{i+1}^{-1}\big)\Ab_{i+1}^{k_{i+1}}\ldots\Ab_{t}^{K_t}\xb_t^{*}
    &=\xb_t^{*\top}\Ab_{t}^{K_t}\ldots\Ab_{i+1}^{k_{i+1}}\frac{\Vb_i^{-1}\xb_{i}\xb_i^{\top}\Vb_i^{-1}}{1+\|\xb_i\|_{\Vb_i^{-1}}^2}\Ab_{i+1}^{k_{i+1}}\ldots\Ab_{t}^{K_t}\xb_t^{*}\notag\\
    &\leq \big(\xb_t^{*\top}\Ab_{t}^{K_t}\ldots\Ab_{i+1}^{k_{i+1}}\Vb_i^{-1}\xb_{i}\big)^2\notag\\
    &\leq\|\Ab_{t}^{K_t}\ldots\Ab_{i+1}^{k_{i+1}}\Vb_i^{-1/2}\xb_t^*\|_2^2\cdot\|\Vb_i^{-1/2}\xb_i\|_2^2\notag\\
    &\leq \prod_{j=i+1}^{t}(1-2\eta_j\lambda_{\min}(\Vb_j))^{2K_j}\|\xb_i\|_{\Vb_i^{-1}}^2\|\xb_t^*\|_{\Vb_i^{-1}}^2,\notag
\end{align}
where we used $0<1/i\leq\|\xb_t^*\|_{\Vb_i^{-1}}\leq1$. Therefore, we have
\begin{align}
    \xb_t^{*\top}\bSigma_{t,K_t}\xb_t^* &\geq  \frac{1}{2\beta_T}\xb_t^{*\top}\Vb_t^{-1}\xb_t^*-\frac{1}{2\beta_T}\prod_{i=1}^{t}(1-2\eta_i\lambda_{\min}(\Vb_i))^{2K_i}\|\xb_t^*\|_{\Vb_1^{-1}}^2\notag\\
    &\qquad-\frac{1}{2\beta_T}\sum_{i=1}^{t-1}\prod_{j=i+1}^{t}(1-2\eta_j\lambda_{\min}(\Vb_j))^{2K_j}\|\xb_i\|_{\Vb_i^{-1}}^2\|\xb_t^*\|_{\Vb_i^{-1}}^2.
\end{align}
Similar to the proof of Lemma~\ref{lemma:concentration_events_sampling}, when we choose  $K_j\geq \kappa_j\log(3\sqrt{t})$, we have
\begin{align}\label{eq:optimisim_denominator}
   \|\xb_t^*\|_{\bSigma_{t,K_t}} &\geq\frac{1}{2\beta_T}\bigg(\|\xb_t^*\|_{\Vb_{t}^{-1}}-\frac{\|\xb_t^*\|_2}{(3\sqrt{t})^t}-\sum_{i=1}^{t-1}\frac{1}{(3\sqrt{t})^{t-i}}\|\xb_t^*\|_2\bigg)\notag\\
   &\geq\frac{1}{2\beta_T}\bigg(\|\xb_t^*\|_{\Vb_{t}^{-1}}-\frac{1}{3\sqrt{t}}\|\xb_t^*\|_2-\frac{1}{6\sqrt{t}}\|\xb_t^*\|_2\bigg)\notag\\
   &\geq \frac{1}{4\beta_T}\|\xb_t^*\|_{\Vb_{t}^{-1}},
\end{align}
where we used the fact that $\lambda_{\min}(\Vb_{t}^{-1})\geq 1/t$.

Therefore, according to \eqref{eq:optimisim_numerator} and \eqref{eq:optimisim_denominator}, it holds that
\begin{align}
    |Z_t|=\bigg|\frac{\xb_{a_t}^{\top}\btheta^*-\xb_{a_t}^{\top}\bmu_{t,k}}{\sqrt{\xb^{\top}\bSigma_{t,k}\xb}}\bigg|\leq\frac{R\sqrt{d\log(t^3/\delta)}+2}{1/(4\beta_T)},
\end{align}
which implies $|Z_t|<1$ when $\beta_t^{-1}=4R\sqrt{d\log\frac{t^3}{\delta}}+8$. This completes our proof.
\end{proof}

\subsection{Proof of Lemma~\ref{lemma:prob_pull_unsaturated}}
\begin{proof}[Proof of Lemma~\ref{lemma:prob_pull_unsaturated}]
Since the algorithm chooses the arm to pull based on the estimated reward $\xb^{\top}\btheta_{t,k}$, as long as we can find an arm in the unsaturated set that beats all arms in the saturated set, we will have $\xb_t\in\cU_t$. Recall the definition in \eqref{eq:def_unsaturated}, we know that the best arm is in the unsaturated set, i.e., $\xb_t^*\in \cU_t$. Therefore, it holds that 
\begin{align}
   \{\xb_t\in\cU_t\}\supseteq \big\{\xb_{t}^{*\top}\btheta_{t,k}>\xb^{\top}\btheta_{t,k}, \forall\xb\in\cS_t\big\}.
\end{align}
Conditional on event $E_{R,t}$, we have
\begin{align}
    \PP\big(\xb_t^{*\top}\btheta_{t,k}>\xb_{t}^{*\top}\btheta^*\big)&= \PP\big(\xb_t^{*\top}\btheta_{t,k}>\xb_{t}^{*\top}\btheta^*|E_{S,t}\big)\PP(E_{S,t})+\PP\big(\xb_t^{*\top}\btheta_{t,k}>\xb_{t}^{*\top}\btheta^*|E_{S,t}^c\big)\PP(E_{S,t}^c)\notag\\
    &\leq \PP\big(\xb_t^{*\top}\btheta_{t,k}>\xb_{t}^{*\top}\btheta^*|E_{S,t}\big)+\PP(E_{S,t}^c).
\end{align}
Recall the definition of the gap and the saturated set in \eqref{eq:def_saturated}, we have that $\xb_{t}^{*\top}\btheta^*=\xb^{\top}\btheta^*+\Delta(t)(\xb)\geq \xb^{\top}\btheta^* +g_t(\xb)$ for any $\xb\in\cS_t$, where $g_t(\xb)$ is defined as in \eqref{def:choice_of_saturated_threshold}. Then it holds that
\begin{align}
   \PP\big(\xb_t^{*\top}\btheta_{t,k}>\xb_{t}^{*\top}\btheta^*|E_{S,t}\big)&\leq \PP\big(\xb_t^{*\top}\btheta_{t,k}>\xb^{\top}\btheta^*+g_t(\xb), \forall\xb\in\cS_t|E_{S,t}\big)\notag\\
   &\leq\PP\big(\xb_t^{*\top}\btheta_{t,k}>\xb^{\top}\btheta_{t,k}, \forall\xb\in\cS_t|E_{S,t}\big),
\end{align}
where the second inequality is true  since $|\xb^{\top}(\btheta_{t,k}-\btheta^*)|\leq g_t(\xb)$ based on the definition of events $E_{R,t}$ and $E_{S,t}$ in \eqref{def:event_regression} and \eqref{def:event_sampling} respectively.  Therefore, we have
\begin{align*}
    \PP(\xb_t\in\cU_t)&\geq\PP\big(\xb_t^{*\top}\btheta_{t,k}>\xb_{t}^{*\top}\btheta^*|E_{S,t}\big)\notag\\
    &\geq \PP\big(\xb_t^{*\top}\btheta_{t,k}>\xb_{t}^{*\top}\btheta^*\big)-\PP(E_{S,t}^c(t))\notag\\
    &\geq \frac{1}{2\sqrt{2e\pi}}-\frac{1}{t^2},
\end{align*}
where the last inequality holds due to Lemma~\ref{lemma:concentration_events_sampling} and Lemma~\ref{lemma:optimism}.
\end{proof}

\section{Auxiliary Lemmas}
\begin{lemma}\label{lem:gaussian-tail}\citep{abramowitz1964handbook}
Suppose $Z$ is a Gaussian random variable $Z\sim\cN(\mu,\sigma^2)$, where $\sigma>0$. For $0\leq z\leq 1$, we have
\begin{align*}
    \PP(Z>\mu+z\sigma)\geq \frac{1}{\sqrt{8\pi}} e^{-\frac{z^2}{2}}, \quad
     \PP(Z<\mu-z\sigma)\geq \frac{1}{\sqrt{8\pi}} e^{-\frac{z^2}{2}}.
\end{align*}
And for $z\geq 1$, we have
\begin{align*}
    \frac{e^{-z^2/2}}{2z \sqrt{\pi}} \leq \PP(|Z-\mu|>z\sigma)\leq \frac{e^{-\frac{z^2}{2}}}{z\sqrt{\pi}}  .
\end{align*}
\end{lemma}
\begin{lemma}\label{lemma:subgaussian_norm_bound}\citep{vershynin2010introduction}
Let $\bX$ be a $\sigma^2$-subGaussian vector in $\RR^d$. Then we have $\EE[\|\bX\|_2]\leq 4\sigma\sqrt{ d}$. For $\delta\in(0,1)$, with probability at least $1-\delta$ that $\|\bX\|_2\leq 4\sigma\sqrt{d}+2\sigma\sqrt{\log(1/\delta)}$.
\end{lemma}

The following lemma introduces the Azuma-Hoeffding inequality for super-martingale. 
\begin{lemma}\label{lemma:azuma_martingale}
Suppose $\{X_k\}_{k=0,1,\ldots}$ is a super-martingale and satisfies $|X_{k+1}-X_{k}|<c_{k+1}$ for all $k\geq 0$. Then for any $\epsilon>0$, we have 
\begin{align*}
    \PP(X_T-X_0\geq \epsilon)\leq\exp\bigg(-\frac{\epsilon^2}{2\sum_{t=1}^{T}c_{t}^2}\bigg).
\end{align*}
\end{lemma}


\end{document}